\documentclass{article}
\usepackage{arxiv_preprint,times}

\usepackage{amsmath,amsfonts,bm}

\def\eqref#1{equation~\ref{#1}}
\def\1{\bm{1}}

\DeclareMathAlphabet{\mathsfit}{\encodingdefault}{\sfdefault}{m}{sl}
\SetMathAlphabet{\mathsfit}{bold}{\encodingdefault}{\sfdefault}{bx}{n}

\newcommand{\Var}{\mathrm{Var}}

\newcommand{\Cov}{\mathrm{Cov}}
\DeclareMathOperator*{\argmin}{arg\,min}

\usepackage{amsmath,amssymb,amsthm}
\usepackage{mathtools}
\usepackage{booktabs}
\usepackage{longtable}
\usepackage{multirow}
\usepackage{graphicx}
\usepackage{wrapfig}
\usepackage[caption=false]{subfig}
\graphicspath{{figures/}}
\usepackage{xcolor}
\usepackage{colortbl,tabularx}
\definecolor{tablerow}{gray}{0.95}
\definecolor{bestblue}{RGB}{31,95,159}
\definecolor{secondorange}{RGB}{224,123,57}

\usepackage{float}
\usepackage[hidelinks]{hyperref}
\usepackage{cleveref}
\usepackage{enumitem}
\usepackage{xspace}
\usepackage{pdflscape}
\usepackage{rotating}

\newtheorem{theorem}{Theorem}[section]
\newtheorem{proposition}[theorem]{Proposition}
\newtheorem{corollary}[theorem]{Corollary}

\theoremstyle{definition}
\newtheorem{definition}[theorem]{Definition}
\theoremstyle{remark}

\theoremstyle{definition}
\newtheorem{assumption}[theorem]{Assumption}
\crefname{assumption}{assumption}{assumptions}
\Crefname{assumption}{Assumption}{Assumptions}

\newcommand{\pareneqref}[1]{\textup{(\ref{#1})}}

\newcommand{\Ebar}{\bar{r}}

\newcommand{\Jdyn}{\mathcal{J}_{\text{dyn}}}
\newcommand{\Juni}{\mathcal{J}_{\text{uni}}}
\newcommand{\Daux}{S}
\newcommand{\LL}{\mathcal{L}}
\newcommand{\Corr}{\mathrm{Corr}}

\newenvironment{theoryrevision}{\par\begingroup}{\par\endgroup}
\DeclareRobustCommand{\theoryrevisioninline}[1]{#1}
\DeclareRobustCommand{\framingrevision}[1]{#1}
\title{Adaptive Patching Is Harder Than it Looks for Time-Series Forecasting}

\author{%
  \begin{minipage}{\textwidth}
    \centering\normalfont\small
    \textbf{Federico Zucchi}\textsuperscript{1}\quad
    \textbf{Yi Xie}\textsuperscript{2}\quad
    \textbf{Chao Zhang}\textsuperscript{3}\quad
    \textbf{Thomas Lampert}\textsuperscript{1}\quad
    \textbf{Ziyue Li}\textsuperscript{2,5}\\[5pt]
    \textsuperscript{1}ICube, University of Strasbourg, Illkirch-Graffenstaden, France\\
    \textsuperscript{2}Technical University of Munich\\
    \textsuperscript{3}FinTech Thrust, The Hong Kong University of Science and Technology (Guangzhou)\\
    \textsuperscript{5}Heilbronn Data Science Center, Munich Data Science Institute\\[5pt]
    \texttt{f.zucchi@unistra.fr}\quad
    \texttt{aaron.xie@tum.de}\quad
    \texttt{chaoz@hkust-gz.edu.cn}\\
    \texttt{lampert@unistra.fr}\quad
    \texttt{ziyue.li@tum.de}
  \end{minipage}%
}
\date{}
\hypersetup{%
  pdftitle={Adaptive Patching Is Harder Than it Looks for Time-Series Forecasting},
  pdfauthor={Federico Zucchi, Yi Xie, Chao Zhang, Thomas Lampert, Ziyue Li}
}

\begin{document}
\maketitle

% ============================================================================
% ABSTRACT
% ============================================================================
\begin{abstract}
Adaptive patching has been proposed for time-series forecasting: allocate finer
patches where the sequence looks locally informative. We compare three adaptive
patching methods with tuned uniform counterparts that use one fixed patch size
throughout the input, while preserving each method's backbone
family and shared training protocol.
Across the evaluated long-horizon benchmarks, we observe mostly
small accuracy differences and no consistent advantage for adaptive patching;
uniform variants generally train faster and avoid adaptive routing at inference.
Motivated by these observations, we develop a theoretical framework for budgeted
patch allocation through a common resolution-density representation of uniform,
nonuniform, and overlapping layouts, independently of the patching rule. Within
a separable shared-loss model, we derive an exact criterion for when nonuniform
allocation improves on uniform allocation at the same budget. We characterize
modeled gains locally through a quadratic surrogate and bound them globally
under strong convexity and a bounded loss derivative. The analysis distinguishes
where resolution is placed from how much it is redistributed: even an allocation
that correctly ranks regions by their weights can redistribute too much
resolution and increase modeled loss. Without coupling across time, positive
regional weights preserve the minimizers of the shared loss. When the shared
loss has a common minimizing rate, the remaining absolute headroom for
improvement tends to zero as the uniform rate approaches it. These results
characterize the roles of regional weighting, placement, and redistribution
strength under the stated model. The empirical findings support evaluating
adaptive patching against a tuned uniform baseline and weighing any accuracy
gain against the additional training and inference cost of adaptive routing.
\end{abstract}

% ============================================================================
% 1. INTRODUCTION
% ============================================================================
\section{Introduction}
\label{sec:intro}
\begin{wrapfigure}[15]{r}{0.50\textwidth}
  \vspace{-\intextsep}
  \centering
  \includegraphics[width=\linewidth]{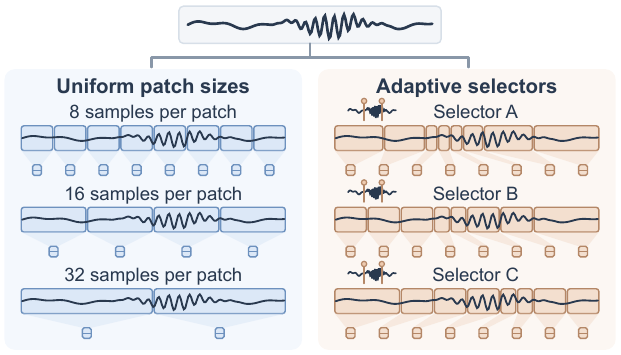}
  \begingroup
  
  \caption{\textbf{Patching choices.} Uniform lengths versus input-adaptive boundaries.}
  \label{fig:patching-choices}
  \endgroup
  \vspace{-5pt}
\end{wrapfigure}
\framingrevision{Patching is widely used to tokenize time series for
Transformer-based forecasting. As in vision
Transformers~\citep{dosovitskiy2021image}, it groups neighboring observations
into tokens, providing the representation used by
PatchTST~\citep{nie2023patchtst} and subsequent forecasting foundation models
such as TimesFM and Moirai~\citep{timesfm,moirai}. But a fixed patch length
assigns the same temporal granularity throughout the input. In vision, a line
of work has explored alternatives that vary input resolution, patch
granularity, or token computation to use resources more
efficiently~\citep{dehghani2023navit,ronen2023apsvit,yin2022avit,bolya2023tome,dart2025}.
Carrying this idea over to time series sounds reasonable: change points,
regime shifts, and bursts of high-frequency structure look meaningful in the
input, making them natural candidates for finer patches. Recent dynamic
patching methods pursue this intuition by placing boundaries or choosing
patch sizes using predictive entropy, temporal heterogeneity, learned
information density, or mixtures of patch
sizes~\citep{abeywickrama2025entrope,ding2026timemosaic,huang2024hdmixer,ankireddy2026timesqueeze,feng2025kairos}.
\Cref{fig:patching-choices} illustrates this allocation choice: uniform
patching uses a fixed patch length throughout the input, whereas adaptive
patching varies boundaries to assign finer resolution to selected regions.}

\framingrevision{The premise is harder to defend than it looks, because
``looks informative'' and ``finer patching helps the loss'' are two different
statements. A measure of local structure can guide a patching rule, but its
usefulness depends on whether the resulting allocation improves forecasting.
Pointwise losses such as MSE do not by themselves specify where extra
resolution should be spent: a high-variance region need not be one whose
error finer patches can reduce. Under a fixed resolution budget, refinement
also requires coarsening elsewhere. The question is therefore both where to
place finer patches and how much resolution to redistribute.}

\framingrevision{These considerations raise two questions. Empirically, does
adaptive patching improve forecasting enough to justify its additional
training and inference cost relative to a carefully tuned uniform
alternative? Mathematically, how can different patch layouts be represented
and compared under a common resolution budget? Addressing the first requires
comparisons within the same model family and training protocol. Addressing
the second requires a formulation that separates the patch layout from the
rule that generates it and the objective used to assess it.}

\begin{figure*}[t]
  \centering
  \subfloat[\label{fig:hero-mse}]{\includegraphics[width=0.325\textwidth,trim=0 6 0 15,clip]{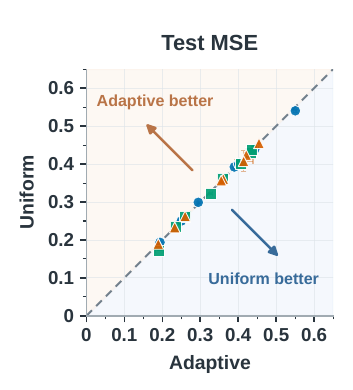}}\hfill
  \subfloat[\label{fig:hero-training}]{\includegraphics[width=0.325\textwidth,trim=0 6 0 15,clip]{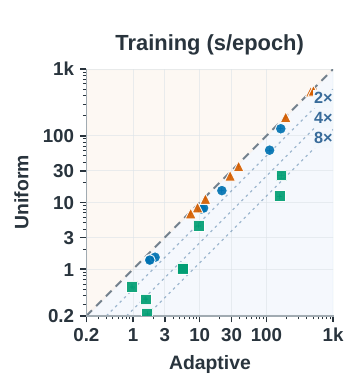}}\hfill
  \subfloat[\label{fig:hero-inference}]{\includegraphics[width=0.325\textwidth,trim=0 6 0 15,clip]{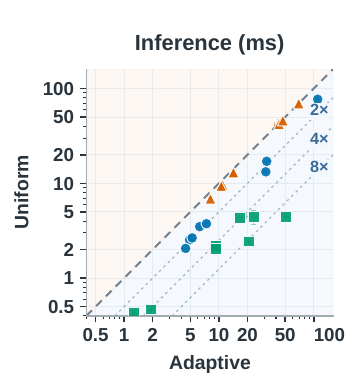}}\par
  \includegraphics[width=0.54\textwidth,trim=0 6 0 4,clip]{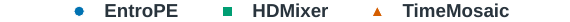}
  \caption[Validation-selected uniform versus adaptive patching]{\textbf{Validation-selected uniform versus adaptive patching.}
  Each point represents one method and dataset, averaging over seeds
  and forecast horizons. Colors and markers identify methods.
  Points below equality favor uniform;
  $2\times$, $4\times$, and $8\times$ guides show uniform speedups.
  Training cost is steady-state time in seconds per epoch and inference cost is batch latency;
  runtime axes are logarithmic.}
  \label{fig:hero}
\end{figure*}

\paragraph{This paper.}
We compare independently validation-selected adaptive and uniform variants
of EntroPE~\citep{abeywickrama2025entrope},
HDMixer~\citep{huang2024hdmixer}, and TimeMosaic~\citep{ding2026timemosaic}
within each method's backbone family and training protocol. Across the
evaluated settings, $79.2\%$ of observed MSE differences lie within $\pm2\%$;
uniform has lower mean test MSE in $56.3\%$, $71.9\%$, and $37.5\%$ of settings,
respectively, while reducing measured training and inference time
(\Cref{fig:hero,tab:main}).
The larger accuracy differences are method- and dataset-specific.
DLinear~\citep{zeng2023transformers} and TSMixer~\citep{chen2023tsmixer}
established the importance of simple forecasting baselines. Our comparison
examines the corresponding choice at the tokenization stage.

\framingrevision{Motivated by these observations, we develop, to our knowledge,
the first general theoretical framework for time-series patch allocation
through a common resolution-density representation. The representation
accommodates uniform, nonuniform, and overlapping patch layouts under a fixed
nominal budget. Its generality lies in separating how patches distribute
resolution across the input from both the patching rule and the objective
used to evaluate the allocation.}

\framingrevision{As one instantiation, we analyze separable objectives in which
regional weights $K_t$ scale a shared loss response $D(r_t)$ to resolution.
An exact improvement criterion separates an allocation-dependent alignment
term from a dispersion term that is nonnegative under convexity. A local
quadratic surrogate characterizes the roles of alignment and redistribution
strength, and a global bound controls gains under strong-convexity
assumptions. The analysis admits both improving reallocations and uniform
optimality: when the mean resolution equals the shared loss-minimizing rate,
uniform allocation is optimal even with heterogeneous regional weights
(\Cref{sec:theory}). These results distinguish regional importance from
preferred resolution and show why identifying where to refine leaves the
amount of refinement to be determined.}

\paragraph{Contributions.}
\begin{itemize}
  \item \framingrevision{A \emph{validation-selected comparison} of adaptive
    and uniform patching in EntroPE, TimeMosaic, and HDMixer across eight
    datasets and four forecast horizons per method, with independent
    selection of both variants and evaluation on fresh seeds. Most MSE
    differences are small, while uniform variants reduce measured training
    and inference costs; the largest accuracy effects remain method- and
    dataset-specific (\Cref{fig:hero,tab:main}).}
  \item \framingrevision{A \emph{general theoretical framework for time-series
    patch allocation} that represents uniform, nonuniform, and overlapping
    layouts through a common resolution density, independently of the
    patching rule. Within a separable shared-loss model, we derive an exact
    criterion for improvement over uniform allocation, characterize gains
    locally, obtain a global gain bound under strong-convexity assumptions,
    and establish uniform optimality when the mean resolution equals the
    common loss-minimizing rate (\Cref{sec:theory}).}
\end{itemize}

% ============================================================================
% 2. RELATED WORK
% ============================================================================
\section{Related work}
\label{sec:related}

\paragraph{Uniform patching for time-series Transformers.}
PatchTST~\citep{nie2023patchtst} established that partitioning a length-$T$ series
into $N = \lceil T/b \rceil$ patches of fixed length $b$ reduces
attention cost from $\mathcal{O}(T^{2})$ to $\mathcal{O}(N^{2})$ while providing
implicit local smoothing; iTransformer~\citep{itransformer} kept a fixed-token
convention while reorganising what each token represents.
Subsequent foundation-model work
(TimesFM~\citep{timesfm}, Chronos~\citep{ansari2024chronos}, Moirai~\citep{moirai})
mostly retained static patching or tokenization choices as the default.
Our empirical comparison uses a validation-selected uniform patch size for
each method, dataset, and horizon, following the emphasis on simple
baselines in DLinear and TSMixer~\citep{zeng2023transformers,chen2023tsmixer}.

\paragraph{Dynamic patching and adaptive tokenization in time series.} Recent
methods place boundaries adaptively or vary granularity using a range of proxies:
EntroPE~\citep{abeywickrama2025entrope} uses predictive entropy;
TimeSqueeze~\citep{ankireddy2026timesqueeze} uses an SSM encoder to measure
information density; TimeMosaic~\citep{ding2026timemosaic} uses temporal
heterogeneity; PathFormer~\citep{chen2024pathformer}, DualSG~\citep{ding2025dualsg},
PatchMLP~\citep{patchmlp}, and Kairos~\citep{feng2025kairos} route or mix
multiple granularities; and HDMixer~\citep{huang2024hdmixer} mixes hierarchical,
length-variable patches. ReinPatch~\citep{wu2026reinpatch} instead learns patch
boundary placement as a reinforcement-learning policy jointly with the downstream
sequence backbone. A related but distinct line reshapes the
training objective or patch-specific processing rather than the boundaries
themselves~\citep{kudrat2025patch,hu2025timefilter}.

\paragraph{Adaptive tokenization and computation.}
Vision transformers provide the empirical precedent for adaptive
tokenization. NaViT~\citep{dehghani2023navit} processes images at their
native resolution by packing variable-length token sequences; adaptive
patch-size ViTs~\citep{ronen2023apsvit} learn spatially-varying granularity;
A-ViT~\citep{yin2022avit} halts computation per token according to an
informativeness criterion; and Token Merging~\citep{bolya2023tome} compresses
redundant tokens post-hoc. The same idea has a longer history in \emph{adaptive computation}: Adaptive
Computation Time lets recurrent networks learn input-dependent computation
steps~\citep{graves2016act}; PonderNet learns a probabilistic halting
distribution~\citep{banino2021pondernet}; sparse Mixture-of-Experts routes
examples or tokens to a subset of experts~\citep{shazeer2017moe,fedus2022switch};
and Mixture-of-Depths routes tokens through a subset of Transformer
blocks~\citep{raposo2024mod}. These works motivate examining both the
task benefit and the execution cost of an adaptive routing or halting policy.

% ============================================================================
% 3. EMPIRICAL STUDY
% ============================================================================
\section{Empirical study: adaptive versus tuned uniform patching}
\label{sec:experiments}

\subsection{Comparison and selection protocol}
\label{sec:setup-exp}

\begin{figure}[htbp]
  \centering
  \includegraphics[width=0.92\linewidth,trim=0 24 0 62,clip]{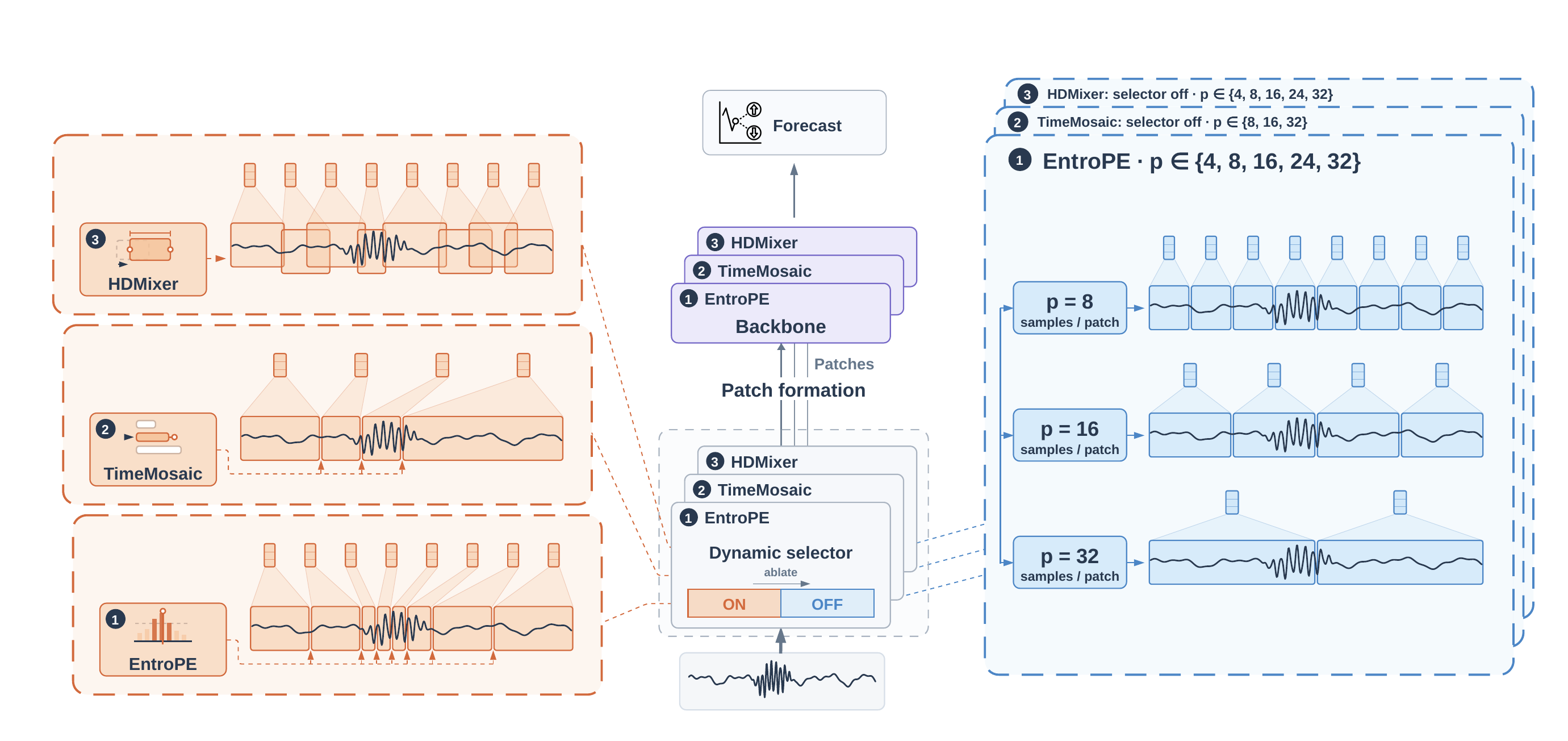}
  \caption{\textbf{Ablation on the dynamic methods.} For each method, the
  adaptive component is replaced by a fixed patch-length sweep.}
  \label{fig:ablation-protocol}
\end{figure}

We compare each adaptive method with a counterpart using one fixed patch
length, select both configurations independently on validation data, and
evaluate their test accuracy and training and inference costs on fresh runs
(\Cref{fig:ablation-protocol}). We evaluate three methods:
\textbf{EntroPE}~\citep{abeywickrama2025entrope}, which places boundaries
using predictive entropy; \textbf{TimeMosaic}~\citep{ding2026timemosaic},
which routes patch sizes using temporal heterogeneity; and
\textbf{HDMixer}~\citep{huang2024hdmixer}, which uses length-extendable
patching. The backbone,
prediction-head construction, optimizer, training budget, and preprocessing
follow the matched implementation (\Cref{app:model-changes}).
\framingrevision{This comparison evaluates complete selected systems: token
counts, parameter counts, and nominal resolution budgets can differ between
arms.} Horizon grids,
patch-size grids, seeds, and selection rules are reported in
\Cref{app:hyper}; \Cref{app:full-tables} gives the complete selected results.
We evaluate eight standard real-world multivariate forecasting benchmarks:
ETT (ETTh1, ETTh2, ETTm1, ETTm2) and Weather~\citep{zhou2021informer}, and
Electricity, Traffic, and Exchange~\citep{lai2018modeling}.

\paragraph{Metrics and aggregation.}
Let $D_{dh}$ and $U_{dh}$ denote mean test MSE over seeds for dataset $d$
and horizon $h$. We report
$G_{dh}=100(D_{dh}-U_{dh})/D_{dh}$; positive gains favor uniform.
Dataset-level gains first average each arm across horizons and then apply
the same formula. Training cost is measured from timed optimization epochs
after warm-up and reported in seconds per epoch. Inference cost is batch
latency; speedup is adaptive/uniform mean time, so values above one favor
uniform.
Measurement details are given in \Cref{app:hyper}.

\subsection{Accuracy and computational cost}
\label{sec:main-result}

\begin{table}[!htbp]
\centering
\caption{\textbf{Dataset-level accuracy and efficiency after validation selection.} Relative test-MSE gain $G=100(D-U)/D$ (\%; positive favors uniform) and adaptive/uniform training and inference speedups. Table entries average seeds and horizons before forming ratios. Bold speedups exceed $1\times$; colored gains show direction, not significance. Details are in \Cref{app:hyper}.}\label{tab:main}
\setlength{\tabcolsep}{2.5pt}
\renewcommand{\arraystretch}{0.85}
\begin{tabular}{@{}lrrrrrrrrr@{}}
\toprule
 & \multicolumn{3}{c}{\textbf{EntroPE}} & \multicolumn{3}{c}{\textbf{HDMixer}} & \multicolumn{3}{c}{\textbf{TimeMosaic}} \\
\cmidrule(lr){2-4}\cmidrule(lr){5-7}\cmidrule(lr){8-10}
Dataset & $G$ & $S_{\mathrm{tr}}$ & $S_{\mathrm{inf}}$ & $G$ & $S_{\mathrm{tr}}$ & $S_{\mathrm{inf}}$ & $G$ & $S_{\mathrm{tr}}$ & $S_{\mathrm{inf}}$ \\
\midrule
ETTh1 & $\textcolor{bestblue}{+1.00}$ & $\mathbf{1.29}$ & $\mathbf{1.79}$ & $\textcolor{bestblue}{+1.97}$ & $\mathbf{4.43}$ & $\mathbf{3.87}$ & $\textcolor{secondorange}{-0.66}$ & $\mathbf{1.07}$ & $\mathbf{1.09}$ \\
\rowcolor{tablerow}
ETTh2 & $\textcolor{bestblue}{+0.34}$ & $\mathbf{1.42}$ & $\mathbf{1.97}$ & $\textcolor{bestblue}{+2.02}$ & $\mathbf{7.62}$ & $\mathbf{8.45}$ & $\textcolor{secondorange}{-0.14}$ & $\mathbf{1.09}$ & $\mathbf{1.12}$ \\
ETTm1 & $\textcolor{secondorange}{-0.67}$ & $\mathbf{1.36}$ & $\mathbf{1.94}$ & $\textcolor{secondorange}{-0.53}$ & $\mathbf{5.28}$ & $\mathbf{4.24}$ & $\textcolor{secondorange}{-0.44}$ & $\mathbf{1.14}$ & $\mathbf{1.18}$ \\
\rowcolor{tablerow}
ETTm2 & $\textcolor{secondorange}{-1.45}$ & $\mathbf{1.38}$ & $\mathbf{1.96}$ & $\textcolor{secondorange}{-0.54}$ & $\mathbf{5.47}$ & $\mathbf{4.55}$ & $\textcolor{secondorange}{-0.91}$ & $\mathbf{1.10}$ & $\mathbf{1.13}$ \\
Electricity & $\textcolor{bestblue}{+0.15}$ & $\mathbf{1.28}$ & $\mathbf{1.43}$ & $\textcolor{bestblue}{+10.26}$ & $\mathbf{12.63}$ & $\mathbf{11.57}$ & $\textcolor{secondorange}{-0.12}$ & $\mathbf{1.00}$ & $\mathbf{1.00}$  \\
\rowcolor{tablerow}
Exchange & $\textcolor{bestblue}{+0.28}$ & $\mathbf{1.30}$ & $\mathbf{2.16}$ & $\textcolor{bestblue}{+0.94}$ & $\mathbf{1.79}$ & $\mathbf{2.95}$ & $\textcolor{bestblue}{+1.59}$ & $\mathbf{1.07}$ & $\mathbf{1.09}$ \\
Traffic & $\textcolor{bestblue}{+1.85}$ & $\mathbf{1.83}$ & $\mathbf{2.36}$ & $\textcolor{secondorange}{-0.13}$ & $\mathbf{6.57}$ & $\mathbf{5.25}$ & $\textcolor{bestblue}{+0.03}$ & $\mathbf{1.06}$ & $\mathbf{1.02}$ \\
\rowcolor{tablerow}
Weather & $\textcolor{bestblue}{+0.08}$ & $\mathbf{1.42}$ & $\mathbf{1.87}$ & $\textcolor{bestblue}{+0.48}$ & $\mathbf{2.15}$ & $\mathbf{4.15}$ & $\textcolor{secondorange}{-0.17}$ & $\mathbf{1.02}$ & $\mathbf{1.03}$ \\
\midrule
Mean & $\textcolor{bestblue}{+0.20}$ & 1.41 & 1.94 & $\textcolor{bestblue}{+1.81}$ & 5.74 & 5.63 & $\textcolor{secondorange}{-0.10}$ & 1.07 & 1.08 \\
\bottomrule
\end{tabular}
\end{table}

Uniform has lower mean MSE in $56.3\%$, $71.9\%$, and $37.5\%$ of settings for
EntroPE, HDMixer, and TimeMosaic, respectively. Across all setting-level
gains, $79.2\%$ lie within $\pm2\%$ and $92.7\%$ within $\pm5\%$.
The dataset averages in \Cref{fig:hero-mse} cluster near equal test MSE,
with the larger deviations depending on method and dataset.
The full range is $-5.02\%$ (EntroPE, ETTm2, horizon 336) to $+15.15\%$
(HDMixer, Electricity, horizon 720). Dataset-level gains average
$+0.20\%$, $+1.81\%$, and $-0.10\%$ in the same method order
(\Cref{tab:main}). HDMixer on Electricity is the clearest outlier:
uniform reduces horizon-averaged MSE by $10.26\%$. TimeMosaic favors
adaptive on $75\%$ of dataset averages, although those losses for
uniform remain below $1\%$. Thus, the largest uniform advantage is concentrated
in HDMixer--Electricity, while TimeMosaic more often favors adaptive patching.
\Cref{fig:hero-training} shows that most training-time points lie below
equality, indicating faster epochs for the uniform variants.
\Cref{fig:hero-inference} shows the same pattern for batch latency.
\Cref{tab:main} shows that
the magnitude depends strongly on the method. EntroPE averages
$1.41\times$ for training and $1.94\times$ for inference. The corresponding
HDMixer ratios are $5.74\times$ and $5.63\times$; TimeMosaic remains close to
parity at $1.07\times$ and $1.08\times$.

\begin{wrapfigure}[9]{R}{0.35\textwidth}
  \vspace{-\intextsep}
  \centering
  \includegraphics[width=\linewidth]{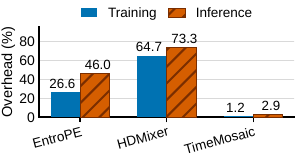}
  \caption{Mean router overhead.}
  \label{fig:router-overhead}
\end{wrapfigure}

\framingrevision{To estimate routing overhead, we first run the adaptive model and save the
patch tokens it produces. We then time a second run with the same checkpoint
and inputs, bypassing adaptive token generation and feeding the saved tokens
to the forecasting backbone.
The measured time reduction (\Cref{fig:router-overhead}) accounts for} $26.6\%$ of EntroPE training time and
$46.0\%$ of its inference time, compared with $64.7\%$ and $73.3\%$ for
HDMixer. TimeMosaic's routing contribution is small at $1.2\%$ and $2.9\%$.
HDMixer's auxiliary objective contributes a further $12.2\%$ of native
adaptive training-step time on average.
Uniform speedups are substantial for HDMixer and small for TimeMosaic,
whose adaptive variant has small accuracy advantages on most datasets.
The practical benefit of removing adaptive patching therefore depends on
the method.

% ============================================================================
% 4. A GENERAL FRAMEWORK FOR BUDGETED PATCH ALLOCATION
% ============================================================================
\section{\framingrevision{A general framework for budgeted patch allocation}}
\label{sec:theory}

\framingrevision{The empirical comparison motivates the study of patch allocation
as a mathematical problem. Inspired by the allocation perspective of
rate--distortion theory~\citep{cover2006elements}, we represent uniform,
nonuniform, and overlapping patch layouts through their resolution densities
under a fixed nominal budget. This representation is independent of the
boundary-generation rule and can be paired with different models of how
resolution affects loss. Here we study separable objectives with regional
weights and a shared loss response to characterize the roles of placement,
dispersion, and budget in reallocation.}

\Cref{fig:concept} previews the exact improvement criterion, its local
quadratic surrogate, and the common-minimum regime.
% Let the model illustrations float so following text can fill the page.
\begin{figure}[htbp]
  \centering
  \subfloat[Exact improvement criterion.\label{fig:allocation-criterion}]{\includegraphics{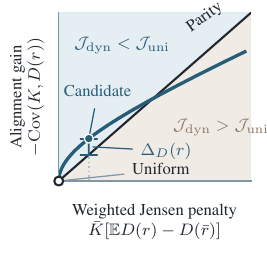}}\hfill
  \subfloat[Local quadratic surrogate.\label{fig:quadratic-surface}]{\includegraphics{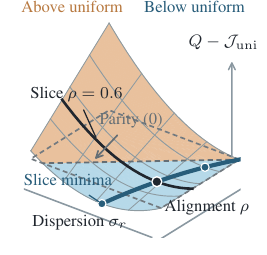}}\hfill
  \subfloat[Common-loss minimum.\label{fig:common-minimum}]{\includegraphics{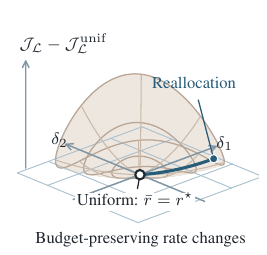}}
  \caption{\framingrevision{\textbf{Budgeted allocation in the shared-loss model.}
  All panels use illustrative objectives.}
  \textbf{(a) Exact improvement criterion, Eq.~\pareneqref{eq:exact-threshold}.}
  Alignment gain $-\Cov(K,D(r))$ must exceed the Jensen penalty
  $\bar K[\mathbb E D(r)-D(\bar r)]$. The vertical gap above the diagonal
  is the gain $\Delta_D(r)$.
  \textbf{(b) Local quadratic surrogate.} For the illustrative surrogate
  $Q-\Juni=\sigma_r^2/2-\rho\sigma_r$ corresponding to
  Eq.~\pareneqref{eq:taylor-threshold}, blue lies below uniform and orange
  above. The black curve fixes alignment $\rho=0.6$; the dark blue curve joins
  minima at positive alignment. \framingrevision{At positive alignment,
  the optimal gain in this surrogate scales with $\rho^2$ by
  Eq.~\pareneqref{eq:delta-max-body}.}
  \textbf{(c) Common-loss minimum, Eq.~\pareneqref{eq:remaining-headroom}.}
  When the mean rate equals the common minimizing rate $r^\star$, uniform
  allocation is globally optimal: $\Delta_{\LL}(r)\leq0$.}
  \label{fig:concept}
\end{figure}

\subsection{Setup}
\label{sec:setup}

\begin{definition}[Nominal resolution density]
\begin{theoryrevision}
\label[definition]{def:bitrate}
Fix an integer \framingrevision{sequence length} $T\geq1$ and a nominal per-token budget $C>0$.
A patching scheme produces $N$ nonempty patches $P_1,\dots,P_N$ covering
(possibly overlapping) intervals $S_i\subseteq\{0,\dots,T-1\}$ of length
$L_i=|S_i|>0$. The \emph{nominal resolution density} at step $t$ is
\begin{equation}
\label{eq:bitrate}
r_{t} \;:=\; \sum_{i=1}^{N} \frac{C}{L_{i}} \,\mathbf{1}\!\left\{t \in S_{i}\right\}
\end{equation}
Each patch contributes $C/L_i$ units uniformly across its span.
\framingrevision{The constant $C$ defines a nominal accounting unit for
comparing resolution across patch layouts.}
\end{theoryrevision}
\end{definition}

\begin{definition}[Budget constraint]
\begin{theoryrevision}
\label[definition]{def:budget}
With common $C$ and fixed $N,T$, every such patching scheme spends the same
total nominal budget: $\sum_{t} r_{t} = NC$. The mean rate
$\Ebar := \mathbb{E}_{t}[r_{t}] = NC/T$ is the same for uniform and nonuniform schemes with those fixed quantities.
\end{theoryrevision}
\end{definition}

\begin{definition}[Feasible allocations and objective]
\begin{theoryrevision}
\label[definition]{def:objective}
For analysis, sample $t$ uniformly from $\{0,\ldots,T{-}1\}$ and write
$K=K_t$, $r=r_t$, and $\mathbb{E}$ for the empirical average over time. Let
$K_t\ge 0$ be regional loss weights with $\bar K:=\mathbb{E}[K]>0$, and let
$I\subset(0,\infty)$ be the admissible rate interval. The continuous
relaxation is $\mathcal{F}(\Ebar,I):=\{r\in I^T:\mathbb{E}[r]=\Ebar\}$,
of which realizable patching policies form a subset. For a shared convex loss
$D:I\to\mathbb{R}$, define the normalized dynamic loss
$\Jdyn(r):=\mathbb{E}[K\,D(r)]$, uniform loss $\Juni:=\bar K\,D(\Ebar)$, and
gain $\Delta_D(r):=\Juni-\Jdyn(r)$.
Positive $\Delta_D$ denotes improvement over uniform allocation.
For the decreasing-loss results, we additionally assume
$D'(r)<0$ and $D''(r)>0$ on the relevant interval.
We set $\sigma_X^2=\mathbb E[(X-\mathbb E X)^2]$ and define
$\Corr(X,Y)=\Cov(X,Y)/(\sigma_X\sigma_Y)$ when both standard deviations
are positive, and zero otherwise. Derivatives are ordinary ambient
derivatives; assumptions at included endpoints refer to an extension to a
neighborhood of the endpoint.
\end{theoryrevision}
\end{definition}

\framingrevision{The regional loss weight $K_t$ has a direct interpretation
within the assumed objective:}
at a uniform rate $\bar r$ with $D'(\bar r)<0$, the marginal reduction
in the regional loss term from additional resolution is
$K_t[-D'(\bar r)]$. Larger weights therefore identify regions with
greater marginal benefit at this operating point.

The comparison holds the weights, loss function, and mean rate fixed.
We ask which feasible allocations satisfy $\Delta_D(r)>0$.

\subsection{An exact criterion for improvement}
\label{sec:threshold}

The covariance decomposition separates the effect of placing resolution
from the effect of varying it across regions.

\begin{theorem}[Exact improvement criterion]
\label{thm:threshold}
Regard $t$ as a uniform random index, and let $K = K_{t}$ and $r = r_{t}$ with
$\mathbb{E}[r] = \Ebar$. Then $\Jdyn(r) < \Juni$ if and only if
\begin{equation}
\label{eq:exact-threshold}
\underbrace{-\Cov\!\bigl(K, D(r)\bigr)}_{\text{alignment gain}}
\;>\;
\mathbb{E}[K]\,\underbrace{\bigl(\mathbb{E}[D(r)] - D(\Ebar)\bigr)}_{\text{Jensen penalty}}.
\end{equation}
\end{theorem}
\begin{proof}See \Cref{app:proof-threshold}.\end{proof}

\begin{theoryrevision}
Convexity makes the right-hand side nonnegative. A candidate allocation
therefore improves the objective only when its alignment term exceeds this
dispersion penalty. The criterion applies to any feasible allocation and
uses no approximation or assumption about the procedure that produced it.
\end{theoryrevision}

\Cref{fig:allocation-criterion} plots this comparison: the vertical gap
between alignment gain and the parity line is $\Delta_D(r)$.
Allocations above the parity line improve on uniform; those below it
increase the objective. The illustrated curve uses $D(r)=1/r$,
$K=(1/2,3/2)$, and $r=(1-u,1+u)$, $0\leq u\leq0.63$,
at fixed mean rate $\bar r=1$.
\paragraph{Allocations independent of the weights.}
\label{sec:tax}

\begin{theoryrevision}
If the rate and weight are independent under the uniform time index, the
alignment term vanishes and only the dispersion penalty remains.
\end{theoryrevision}

\begin{proposition}[Nonimprovement under independence]
\begin{theoryrevision}
\label[proposition]{prop:unaligned}
If $r_t$ is statistically independent of $K_t$, $\mathbb{E}[r]=\Ebar$, and $D$
is convex, then $\Jdyn(r)\geq \Juni$. If $D$ is strictly convex on the interval $I$,
equality holds iff $r_t\equiv\Ebar$ almost surely.
\end{theoryrevision}
\end{proposition}
\begin{proof}See \Cref{app:proof-unaligned}.\end{proof}

\subsection{Improving allocations for decreasing loss}
\label{sec:coupling}

For a decreasing loss, heterogeneous weights permit a small improving
reallocation when the uniform rate is interior to the feasible interval.

\begin{proposition}[Local existence of an improving aligned allocation]
\begin{theoryrevision}
\label[proposition]{prop:local}
Assume $D$ is differentiable at $\Ebar$, $D'(\Ebar)<0$,
$\Var(K)>0$, and $\Ebar \in \mathrm{int}(I)$
(the uniform mean rate lies in the interior of the admissible
interval). For sufficiently small $\varepsilon>0$,
set $r_{t}^{(\varepsilon)} := \Ebar + \varepsilon(K_{t}-\bar{K})$.
Then $\mathbb{E}[r^{(\varepsilon)}]=\Ebar$ and $\Jdyn(r^{(\varepsilon)}) < \Juni$.
\end{theoryrevision}
\end{proposition}
\begin{proof}See \Cref{app:proof-local}.\end{proof}

\framingrevision{The construction gives an improving direction in the
continuous relaxation. Applying it to integer patches involves checking
which rate changes can be realized by admissible boundaries.}
For a strictly convex differentiable loss with $D'(r)<0$, positive
weights, and an interior constrained optimum, \Cref{thm:kkt} gives
$K_tD'(r_t^\star)+\lambda=0$ with $\lambda>0$. Optimal rates then increase
with the weights. The fixed budget couples the regions: increasing one
rate requires reducing another. Without coupling, positive scalar weights
preserve the minimizer of a shared pointwise loss (\Cref{thm:scalar}).

\subsection{Local quadratic analysis}
\label{sec:upper}

A Taylor expansion describes the gain near the uniform allocation.
The exact expansion retains a weighted second-order term and a remainder;
a scalar quadratic surrogate requires an additional approximation.

\begin{theoryrevision}
The exact second-order expansion with remainder is given in
\Cref{prop:taylor}. Neglecting its controlled remainder gives the
second-order criterion
$-D'(\Ebar)\Cov(K,r)>\tfrac12D''(\Ebar)\mathbb{E}[K(r-\Ebar)^2]$.
A further joint-moment approximation $\Cov(K,\delta^2)\approx0$ gives the
scalar surrogate below. Low skewness of $K$ alone does not justify this step:
for the special perturbation $\delta=\varepsilon(K-\bar K)$, the omitted
covariance equals $\varepsilon^2\mathbb E[(K-\bar K)^3]$, but a general
allocation need not have this form. Under the additional approximation,
\end{theoryrevision}
\begin{equation}
\label{eq:taylor-threshold}
\underbrace{-D'(\Ebar)\,\Cov(K, r)}_{\text{alignment gain (linear)}}
\;>\;
\underbrace{\tfrac{1}{2}\,\bar{K}\,D''(\Ebar)\,\Var(r)}_{\text{dispersion penalty (quadratic)}}.
\end{equation}
The right-hand side combines the local curvature with the variance of the
allocation. Equation~\pareneqref{eq:taylor-threshold} describes the surrogate;
the exact improvement criterion remains \Cref{eq:exact-threshold}.

\begin{theorem}[Maximum of the quadratic surrogate]
\label{thm:delta-max}
Let $a=-D'(\Ebar)>0$, $b=D''(\Ebar)>0$, and
$\rho=\Corr(K,r)$, treated as fixed. The local surrogate
$a\rho\sigma_K\sigma_r-\tfrac12\bar K b\sigma_r^2$ has maximum over
$\sigma_r\geq0$
\begin{equation}
\label{eq:delta-max-body}
\Delta_{\max}^{\text{quad}}
\;=\;
\begin{cases}
\dfrac{\bigl(-D'(\Ebar)\bigr)^2\rho^2\sigma_K^2}{2\bar K D''(\Ebar)}, & \rho>0,\\[0.8em]
0, & \rho\le 0,
\end{cases}
\end{equation}
with maximiser
$\sigma_r^\star=(-D'(\Ebar))\rho\sigma_K/(\bar K D''(\Ebar))$ when $\rho>0$.
\end{theorem}
\begin{proof}
The surrogate is a concave quadratic in $\sigma_{r}$ at fixed $\rho,\sigma_{K}$;
read off its maximiser. See \Cref{app:proof-delta-max}.
\end{proof}

\begin{theoryrevision}
Expression~\pareneqref{eq:delta-max-body} maximizes the scalar surrogate over
$\sigma_r\geq0$ while holding its other coefficients fixed. This step also
relaxes the constraints on realizable allocations, so the value need not
be achievable. At fixed remaining coefficients, the surrogate optimum
scales with $\max\{\rho,0\}^2\sigma_K^2$ and inversely with $D''(\Ebar)$.
\end{theoryrevision}

\Cref{fig:quadratic-surface} plots $Q-\Juni=\sigma_r^2/2-\rho\sigma_r$.
The dark blue curve joins minima across positive-alignment slices; the black
curve fixes $\rho=0.6$. Zero dispersion recovers uniform allocation.
\subsection{A global gain bound}
\label{sec:global-bound}

Strong convexity and a bounded derivative give a bound that holds across
the feasible set, without a local approximation.

\begin{theorem}[Global gain bound under strong convexity]
\begin{theoryrevision}
\label{thm:bound}
Assume $D$ is twice differentiable at every point of the interval $I$,
with ordinary ambient derivatives at any included endpoints. Set
$m_D:=\inf_{u\in I}D''(u)>0$, $L_D:=\sup_{u\in I}|D'(u)|<\infty$, and
$\alpha_+(r):=\max\{\Corr(K,-D(r)),0\}$. For any feasible allocation,
\end{theoryrevision}
\begin{equation}
\label{eq:bound-universal}
\Delta_D(r)\le
\alpha_+(r)L_D\sigma_K\sigma_r-\tfrac12\bar K m_D\sigma_r^2
\le \frac{\alpha_+(r)^2L_D^2\sigma_K^2}{2\bar K m_D}
\le \frac{L_D^2\sigma_K^2}{2\bar K m_D}.
\end{equation}
\end{theorem}
\begin{proof}See \Cref{app:proof-bound}.\end{proof}

The allocation-dependent bound uses
$\alpha_+(r)=\max\{\Corr(K,-D(r)),0\}$; the quadratic surrogate instead fixes
$\Corr(K,r)$. The final inequality removes the correlation term and bounds
the gain uniformly over the feasible set.

\subsection{Uniform optimality at a common loss minimum}
\label{sec:trap}

Consider a shared loss $\LL$ with a common global minimum at $r^\star$.
Positive regional weights preserve this minimizing rate. The following
assumption ensures uniqueness; the gain bound itself requires only a common
global minimum and nonnegative weights.

\begin{assumption}[Unique common loss minimum]
\label[assumption]{ass:unique}
$\LL\in C^2(I)$, $\LL''(r)>0$ on $I$, and there exists $r^\star\in I$ with
$\LL'(r^\star)=0$. Hence $r^\star$ is the unique global minimiser of $\LL$.
\end{assumption}

\begin{theoryrevision}
Write $\mathcal{J}_{\LL}(r):=\mathbb{E}[K\LL(r)]$ and
$\mathcal{J}_{\LL}^{\mathrm{unif}}:=\bar K\LL(\Ebar)$.
The exact decomposition and improvement criterion apply with $D$ replaced
by $\LL$; convexity again makes the dispersion penalty nonnegative.
\end{theoryrevision}

\begin{theorem}[\theoryrevisioninline{Uniform optimality and remaining absolute gain}]
\begin{theoryrevision}
\label{thm:trap-global}
Assume $K_t\geq0$ and that $\LL$ has a common global minimiser
$r^\star\in I$. For every feasible allocation,
\end{theoryrevision}
\begin{theoryrevision}
\begin{equation}
\label{eq:remaining-headroom}
\Delta_{\LL}(r)\leq\bar K\,[\LL(\bar r)-\LL(r^\star)],
\qquad
\Delta_{\LL}:=\mathcal J_{\LL}^{\mathrm{unif}}-\mathcal J_{\LL}(r).
\end{equation}
\end{theoryrevision}
\begin{theoryrevision}
In particular, if $\bar r=r^\star$, uniform allocation is globally optimal.
Under \Cref{ass:unique} with $K_t>0$, equality of dynamic and uniform loss
at that budget holds only for $r_t\equiv r^\star$.
\end{theoryrevision}
\end{theorem}
\begin{proof}
\begin{theoryrevision}
The pointwise inequality $\LL(r_t)\geq\LL(r^\star)$ gives
$\mathcal J_{\LL}(r)\geq\bar K\LL(r^\star)$. Subtract it from
$\bar K\LL(\bar r)$. See \Cref{app:proof-trap-global} for equality conditions.
\end{theoryrevision}
\end{proof}
\begin{theoryrevision}
For fixed $K$ and continuous $\LL$, the absolute-gain upper bound tends
to zero as $\bar r\to r^\star$.
\end{theoryrevision}

The local expansion and its remainder bound are given in
\Cref{app:proof-trap}. An optional decomposition of $\LL$ into a base loss
and an auxiliary scalar cost is examined in \Cref{app:proof-locell}.

\Cref{fig:common-minimum} illustrates this result on a surface of
budget-preserving reallocations: uniform lies at the minimum when
$\bar r=r^\star$.
The example uses $\LL(r)=1+(r-1)^2$, $K=(1,2,3)$, and
$\bar r=r^\star=1$. The two free rate changes satisfy
$\delta_3=-\delta_1-\delta_2$, preserving the budget throughout the surface.
\subsection{Allocation regimes}
\label{sec:theory-summary}

The exact criterion and gain analyses distinguish two regimes in the
shared-loss model. Heterogeneous weights permit local improvement for
decreasing loss under the interior conditions of \Cref{prop:local}, whereas
uniform allocation is optimal at a common minimizing rate
(\Cref{thm:trap-global}). The local surrogate and global bound quantify gain
under their respective assumptions.

\section{\framingrevision{Using the allocation framework}}
\label{sec:complexity}

\paragraph{\framingrevision{Interpreting regional weights.}}
\framingrevision{Regional weights let the model represent differences in
sensitivity to resolution. At a uniform rate with $D'(\bar r)<0$, the marginal benefit of
added resolution is $K_t[-D'(\bar r)]$, so $\sigma_K$ measures heterogeneity
in this benefit up to a common factor. Treating $K_t$ as a specified weight
allows that heterogeneity to be studied separately from the allocation.}

\paragraph{\framingrevision{Boundary placement and resolution dispersion.}}
\framingrevision{Boundaries enter through the induced rates $r_t$
(\Cref{def:bitrate}). Two nonoverlapping partitions with the same multiset
of patch lengths have the same rate distribution, budget, and Jensen penalty,
but can have different alignment gains $-\Cov(K,D(r))$. The formulation thus
distinguishes where resolution is placed from how unevenly it is distributed.
Candidate boundary rules can be studied through their induced allocations,
accounting for both refined regions and those coarsened to preserve the budget.}

\paragraph{\framingrevision{Alignment and redistribution strength.}}
\framingrevision{In the two-region example of \Cref{fig:allocation-criterion},
$\Corr(K,r)=1$ along the nonuniform part of the curve, yet excessive
redistribution makes the objective exceed its uniform value. Correctly
ordering the rates by the weights therefore leaves the amount of
redistribution undetermined. The maximizing dispersion of the local surrogate
depends on the loss slope and curvature, regional weights, and the specified
alignment (\Cref{thm:delta-max}).}

\paragraph{\framingrevision{Modeling regional resolution needs.}}
\framingrevision{For a forecasting application, the framework offers a way
to distinguish regional importance from preferred patch resolution. Positive
regional weights represent differences in the scale of a shared loss curve
while preserving its minimizing rate. Under this modeling choice, uniform
allocation is optimal when the mean resolution budget matches that common
preferred rate (\Cref{thm:trap-global}).}

% ============================================================================
% 6. CONCLUSION AND FUTURE WORK
% ============================================================================
\section{Conclusion and future work}
\label{sec:discussion}

\paragraph{\framingrevision{Future work.}}
\framingrevision{Future work could extend the analysis to objectives with
interactions between regions, retaining the resolution-density representation
and fixed budget. In such models, refining one region could increase or
decrease the benefit of refining another. A central question is how the
structure and strength of these interactions change the conditions for
improving reallocations and uniform optimality.}

\framingrevision{The empirical comparison finds mostly small accuracy
differences between adaptive and tuned uniform systems, with generally lower measured
execution costs for uniform variants. These observations support careful
uniform-baseline selection in the evaluated forecasting settings and motivate
a formal study of patch allocation. Our theoretical contribution is a general
framework that represents uniform, nonuniform, and overlapping patch layouts
through nominal resolution under a common budget. \framingrevision{The empirical results support a practical starting point:
tune the uniform patch length carefully and assess any accuracy gain from
adaptive patching alongside its additional routing cost. The shared-loss
analysis adds a design lesson: even an allocation that correctly ranks regions
by importance can redistribute too much resolution, so placement and
refinement strength should be considered together.}}

% ============================================================================
% REFERENCES
% ============================================================================
\label{page:main-end}

\label{page:references-start}
\bibliographystyle{iclr2027_conference}
\bibliography{ref}

% \clearpage

% ============================================================================
% APPENDIX
% ============================================================================

\appendix
\raggedbottom

\section*{Appendix}

\paragraph{Structure.} \Cref{app:proofs} collects the proofs and technical
formal statements supporting the main theory. \Cref{app:stats} defines the
aggregation, paired seed tests, dataset-level uncertainty, and multiplicity
corrections for the selected comparison. \Cref{app:full-tables} reports all
96 selected accuracy and efficiency comparisons, including configuration
choices. \Cref{app:model-changes} describes the uniform adaptations.
\Cref{app:hyper} specifies candidate grids, validation selection, seed
counts, fixed hyperparameters, and measurement protocols;
\Cref{app:datasets} gives the chronological dataset splits.

\section{Proofs}
\label{app:proofs}

The proofs below also include formal statements moved out of the main
theory section for space.

\subsection[Proof of scalar invariance]{Proof of \Cref{thm:scalar} (scalar invariance)}
\label{app:proof-scalar}
\label{sec:scalar}

Without any constraint linking the time steps, a common per-step rate
loss $g$ gives the decoupled problem
\begin{equation}
\label{eq:decouple}
\min_{\{r_t\in I\}} \sum_{t=0}^{T-1} K_t\,g(r_t)
\;=\;
\sum_{t=0}^{T-1} \min_{r_t\in I}\,K_t\,g(r_t).
\end{equation}
For $K_t>0$, each weight rescales the same pointwise loss and preserves its minimizer.

\begin{theorem}[Scalar invariance in the unconstrained regime]
\begin{theoryrevision}
\label{thm:scalar}
Let $K_t>0$ and let the same rate loss $g:I\to\mathbb{R}$ apply at every time
step. If $g$ has a unique minimiser $r^\star\in I$, the unique unconstrained
optimizer is $r_t^\star=r^\star$ for all $t$. If instead $g=D$ is strictly
decreasing on a compact interval $I=[r_{\min},r_{\max}]$, the unconstrained
optimizer is the boundary allocation $r_t^\star=r_{\max}$ for all $t$; on an
interval unbounded above the pure decreasing problem has no finite optimum.
In each case, the set of pointwise minimizers is independent of the positive weights.
\end{theoryrevision}
\end{theorem}
\emph{Intuition.} A positive scalar $K_t$ rescales the loss surface but does not
move its argmin. Thus an unconstrained common-loss objective either selects the
same common minimizer everywhere or, for a decreasing loss on a compact
interval, the same upper rate boundary.

\begin{corollary}[Necessity of coupling or landscape variation]
\begin{theoryrevision}
\label[corollary]{cor:coupling}
Under the positive-weight and common-unique-minimiser assumptions of
\Cref{thm:scalar}, no non-constant allocation is an unconstrained minimiser.
A non-constant optimum must leave that regime. Possible mechanisms include:
\end{theoryrevision}
\begin{enumerate}
  \item \emph{Budget coupling:} \framingrevision{a total-resolution constraint}
    forces time steps into a shared resource problem; under the assumptions of \Cref{thm:kkt}, the KKT condition
    $K_{t}D'(r_{t}^{\star})+\lambda=0$ yields a $K_t$-dependent optimum.
  \item \emph{Time-varying loss landscape:} different time steps have
    different $\ell_{t}(r)$, whose per-step optima may differ
    independently of coupling.
  \item \emph{Additional structural constraints:} token limits, boundary
    constraints, nonlocal dependencies, or computational costs that create
    further coupling.
\end{enumerate}
\begin{theoryrevision}
For decreasing strictly convex $D$, positive heterogeneous $K_t$, and an
interior constrained minimiser, \Cref{thm:kkt} gives strictly ordered optimal
rates. For the different objective $\LL$ with a common global minimiser
$r^\star$, a budget with $\bar r=r^\star$ instead makes uniform allocation
optimal (\Cref{thm:trap-global}).
\end{theoryrevision}
\end{corollary}

\begin{proof}[Proof of \Cref{thm:scalar}]
\begin{theoryrevision}
Without a coupling constraint, the global problem separates into independent
per-step subproblems, as in \Cref{eq:decouple}. For a common per-step loss $g$
and $K_t>0$,
\[
\argmin_{r_t\in I}K_t g(r_t)=\argmin_{r_t\in I}g(r_t),
\]
because multiplication by a positive scalar preserves the argmin. If $g$ has a
unique minimiser $r^\star$, each subproblem therefore selects $r_t^\star=r^\star$.
If $g=D$ is strictly decreasing on a compact interval $[r_{\min},r_{\max}]$, each
subproblem is minimised at $r_{\max}$; on an unbounded interval the decreasing
objective has no finite minimiser when the interval is unbounded above.
The positive weights do not change the pointwise minimizer sets.
\end{theoryrevision}
\end{proof}

\subsection[Proof of nonimprovement under independence]{Proof of \Cref{prop:unaligned} (nonimprovement under independence)}
\label{app:proof-unaligned}

\begin{theoryrevision}
Under independence, $K$ and $D(r)$ are independent, so
\[
\Jdyn(r)=\mathbb{E}[KD(r)]=\bar K\,\mathbb{E}[D(r)].
\]
Because $D$ is convex and $\mathbb{E}[r]=\Ebar$, Jensen's inequality gives
$\mathbb{E}[D(r)]\ge D(\Ebar)$, hence $\Jdyn(r)\ge \bar K D(\Ebar)=\Juni$.
If $D$ is strictly convex on $I$, equality requires
$r_t\equiv\Ebar$ almost surely.
\end{theoryrevision}

\subsection[Proof of exact threshold]{Proof of \Cref{thm:threshold} (exact threshold)}
\label{app:proof-threshold}

Treat $t$ as a uniform random index and write $K = K_{t}$ and $r = r_{t}$. The
dynamic scheme improves on the uniform baseline if and only if
\begin{equation}
\mathbb{E}[KD(r)] < \mathbb{E}[K]D(\Ebar).
\end{equation}
Using the covariance identity,
\begin{equation}
\mathbb{E}[KD(r)] = \Cov(K, D(r)) + \mathbb{E}[K]\,\mathbb{E}[D(r)].
\end{equation}
Substituting and rearranging gives
\begin{equation}
-\Cov(K, D(r)) > \mathbb{E}[K]\bigl(\mathbb{E}[D(r)] - D(\Ebar)\bigr),
\end{equation}
which is exactly \Cref{eq:exact-threshold}.

\subsection[Proof of quadratic-surrogate maximum]{Proof of \Cref{thm:delta-max} (quadratic-surrogate maximum)}
\label{app:proof-delta-max}

Using the local quadratic approximation from \Cref{eq:taylor-threshold}, define
\begin{equation}
\Delta(r)
:=
(-D'(\Ebar))\Cov(K,r)
- \tfrac{1}{2}\bar K D''(\Ebar)\Var(r).
\end{equation}
Write $\Cov(K,r)=\rho\sigma_{K}\sigma_{r}$ and $\Var(r)=\sigma_{r}^{2}$.
Then
\begin{equation}
\Delta(r)
=
(-D'(\Ebar))\rho\sigma_{K}\sigma_{r}
- \tfrac{1}{2}\bar K D''(\Ebar)\sigma_{r}^{2}.
\end{equation}
For fixed $\rho$ and $\sigma_K$, this is a concave quadratic in $\sigma_r$. If
$\rho\le0$, the linear term is non-positive for all $\sigma_r\ge0$, so the
maximum is attained at $\sigma_r=0$ with value $0$. If $\rho>0$, differentiating
with respect to $\sigma_r$ and setting to zero gives
\begin{equation}
\sigma_{r}^{\star}
=
\frac{(-D'(\Ebar))\rho\sigma_{K}}
{\bar K D''(\Ebar)}.
\end{equation}
Substituting $\sigma_{r}^{\star}$ back into $\Delta(r)$ yields
\begin{equation}
\Delta_{\max}^{\text{quad}}
\;=\;
\frac{\bigl(-D'(\Ebar)\rho\sigma_{K}\bigr)^{2}}
{2\bar K D''(\Ebar)},
\end{equation}
which is the positive-correlation case of \Cref{eq:delta-max-body}. This is the
maximum of the quadratic surrogate, not an exact gain bound for arbitrary
allocations.

\subsection[Proof of KKT characterisation]{Proof of \Cref{thm:kkt} (KKT characterisation of the optimal constrained allocation)}
\label{app:proof-kkt}

\begin{theorem}[KKT characterisation of the optimal constrained allocation]
\label{thm:kkt}
Assume $D \in C^{1}$ is strictly convex with $D'(r)<0$ for all $r$, and
$K_{t}>0$. If the minimiser $r^{\star}$ of
$\min_{r \in \mathcal{F}(\Ebar,I)} \mathbb{E}[K\,D(r)]$
is interior, there exists $\lambda>0$ such that
\begin{equation}
\label{eq:kkt}
K_{t}\,D'(r_{t}^{\star}) + \lambda = 0,
\qquad t = 0,\ldots,T{-}1.
\end{equation}
Equivalently $D'(r_{t}^{\star}) = -\lambda/K_{t}$, and the allocation is
\emph{monotone} in the weights: $K_{s} > K_{t} \Rightarrow r_{s}^{\star} > r_{t}^{\star}$.
\end{theorem}

\begin{proof}
The constrained optimisation problem is
\begin{equation}
\min_{\{r_t\} \in \mathcal{F}(\bar r,I)}
\frac{1}{T}\sum_{t=0}^{T-1} K_t D(r_t),
\quad
\mathcal{F}(\bar r,I) := \Bigl\{\{r_t\}\in I^T : \tfrac{1}{T}\sum_{t=0}^{T-1} r_t = \bar r\Bigr\}.
\end{equation}
Form the Lagrangian
\begin{equation}
\mathcal{L}\bigl(\{r_t\},\lambda\bigr)
\;=\;
\frac{1}{T}\sum_{t=0}^{T-1} K_t D(r_t)
+ \lambda\!\left(\frac{1}{T}\sum_{t=0}^{T-1} r_t - \bar r\right).
\end{equation}
\begin{theoryrevision}
The gradient of the mean-rate constraint is nonzero since $T\geq1$.
The Lagrange multiplier rule therefore applies at an interior minimiser.
Differentiating $\mathcal L$ with respect to $r_t$ gives
\end{theoryrevision}
\begin{equation}
\frac{1}{T}\bigl(K_t\,D'(r_t^{\star}) + \lambda\bigr) = 0,
\qquad t = 0,\ldots,T{-}1,
\end{equation}
\begin{theoryrevision}
which is \pareneqref{eq:kkt}.  Since $D'(r) < 0$ for all $r$ and $K_t > 0$, the
equation $K_t D'(r_t^{\star}) = -\lambda$ can only hold with $\lambda > 0$.
Rearranging gives $D'(r_t^\star)=-\lambda/K_t$. A differentiable strictly
convex function on an interval has a strictly increasing derivative, so
\end{theoryrevision}
\begin{equation}
K_s > K_t
\;\Rightarrow\;
\frac{-\lambda}{K_s} > \frac{-\lambda}{K_t}
\;\Rightarrow\;
D'(r_s^{\star}) > D'(r_t^{\star})
\;\Rightarrow\;
r_s^{\star} > r_t^{\star}.
\end{equation}
Uniqueness of $r_t^{\star}$ for each $t$ follows from strict convexity of $D$.
\end{proof}

\subsection{Local existence of an improving aligned allocation}
\label{app:proof-local}

\begin{proof}

\begin{theoryrevision}
Define the perturbed allocation
$r_t^{(\varepsilon)} := \bar r + \varepsilon(K_t - \bar K)$ for $\varepsilon > 0$.
The mean is
$\mathbb{E}[r^{(\varepsilon)}] = \bar r + \varepsilon(\bar K - \bar K) = \bar r$,
so the budget constraint is satisfied. Because $\bar r$ is an interior
point and there are finitely many coordinates, sufficiently small
$\varepsilon>0$ also keeps every $r_t^{(\varepsilon)}$ in $I$.
\end{theoryrevision}

Expand $D(r_t^{(\varepsilon)})$ to first order in $\varepsilon$:
\begin{theoryrevision}
\begin{equation}
D\!\bigl(\bar r + \varepsilon(K_t-\bar K)\bigr)
= D(\bar r) + D'(\bar r)\,\varepsilon(K_t - \bar K) + o(\varepsilon).
\end{equation}
\end{theoryrevision}
Multiplying by $K_t$, summing over $t$, and dividing by $T$:
\begin{theoryrevision}
\begin{equation}
\Jdyn\!\bigl(r^{(\varepsilon)}\bigr)
= D(\bar r)\,\mathbb{E}[K]
+ D'(\bar r)\,\varepsilon\,\mathbb{E}\!\left[K(K-\bar K)\right]
+ o(\varepsilon).
\end{equation}
\end{theoryrevision}
Note that $\mathbb{E}[K(K-\bar K)] = \mathbb{E}[K^2] - \bar K^2 = \Var(K) > 0$ by assumption.
Therefore
\begin{theoryrevision}
\begin{equation}
\Jdyn\!\bigl(r^{(\varepsilon)}\bigr) - \Juni
= D'(\bar r)\,\varepsilon\,\Var(K) + o(\varepsilon).
\end{equation}
\end{theoryrevision}
Since $D'(\bar r) < 0$ and $\Var(K) > 0$, the linear coefficient is strictly negative.
For all sufficiently small $\varepsilon > 0$ the linear term dominates and
$\Jdyn(r^{(\varepsilon)}) < \Juni$.
\end{proof}

\subsection[Exact second-order expansion]{Exact second-order expansion (\Cref{prop:taylor})}
\label{app:proof-taylor}

\begin{proposition}[Exact second-order expansion with remainder]
\begin{theoryrevision}
\label[proposition]{prop:taylor}
Let $\delta_t=r_t-\Ebar$ and assume $D\in C^3$ on an open neighborhood
of the segments joining $\Ebar$ to each $r_t$. With
$R_3(\delta):=D(\Ebar+\delta)-D(\Ebar)-D'(\Ebar)\delta-\tfrac12D''(\Ebar)\delta^2$,
\end{theoryrevision}
\begin{equation}
\label{eq:taylor-exact-main}
\Delta_D(r)
=-D'(\Ebar)\Cov(K,r)-\tfrac12D''(\Ebar)\mathbb{E}[K\delta^2]
-\mathbb{E}[K R_3(\delta)],
\end{equation}
\begin{theoryrevision}
and $\mathbb{E}[K\delta^2]=\bar K\Var(r)+\Cov(K,\delta^2)$. If
$|D'''|\le M_3$ throughout those segments, then
$|\mathbb{E}[K R_3(\delta)]|\le \tfrac{M_3}{6}\mathbb{E}[K|\delta|^3]$.
\end{theoryrevision}
\end{proposition}

\begin{proof}
By Taylor's theorem applied to $D \in C^{3}$:
\begin{equation}
D(r_t) = D(\bar r) + D'(\bar r)\,\delta_t
+ \tfrac{1}{2}D''(\bar r)\,\delta_t^2 + R_3(\delta_t),
\end{equation}
where $\delta_t = r_t - \bar r$ and $R_3(\delta_t) = D(r_t) - D(\bar r)
- D'(\bar r)\delta_t - \tfrac{1}{2}D''(\bar r)\delta_t^2$.
Multiply by $K_t$, sum over $t$, and divide by $T$:
\begin{equation}
\frac{1}{T}\sum_t K_t D(r_t)
= D(\bar r)\,\mathbb{E}[K]
+ D'(\bar r)\,\mathbb{E}[K\delta]
+ \tfrac{1}{2}D''(\bar r)\,\mathbb{E}[K\delta^2]
+ \mathbb{E}[K R_3(\delta)].
\end{equation}
Since $\mathbb{E}[\delta] = \mathbb{E}[r - \bar r] = 0$ and
$\mathbb{E}[K\delta] = \mathbb{E}[K(r-\bar r)]
= \mathbb{E}[Kr] - \bar r\,\mathbb{E}[K] = \Cov(K, r)$,
we obtain
\begin{equation}
\Delta_D(r)
= \Juni - \Jdyn(r)
= -D'(\bar r)\,\Cov(K, r)
- \tfrac{1}{2}D''(\bar r)\,\mathbb{E}[K\delta^2]
- \mathbb{E}[K R_3(\delta)],
\end{equation}
which is \pareneqref{eq:taylor-exact-main}.

For the expansion of $\mathbb{E}[K\delta^2]$, write $K = \bar K + (K - \bar K)$:
\begin{equation}
\mathbb{E}[K\delta^2]
= \bar K\,\mathbb{E}[\delta^2] + \mathbb{E}[(K-\bar K)\delta^2]
= \bar K\,\Var(r) + \Cov(K, \delta^2).
\end{equation}

For the remainder bound: by Taylor's theorem with the Lagrange remainder,
$|R_3(\delta)| \leq \frac{M_3}{6}|\delta|^3$.
Multiplying by $K \geq 0$ and taking expectations gives
$|\mathbb{E}[K R_3(\delta)]| \leq \frac{M_3}{6}\,\mathbb{E}[K|\delta|^3]$.
\end{proof}

\subsection[Proof of global gain bound]{Proof of \Cref{thm:bound} (global gain bound under strong convexity)}
\label{app:proof-bound}

\begin{proof}
Starting from the exact decomposition implicit in \Cref{thm:threshold},
\begin{equation}
\Delta_D(r)
= -\Cov(K, D(r)) - \mathbb{E}[K]\bigl(\mathbb{E}[D(r)] - D(\bar r)\bigr).
\end{equation}

\textbf{Step 1: upper-bounding the covariance term.}
Writing $-\Cov(K, D(r)) = \Cov(K, -D(r))$, we use
$\Cov(K, -D(r)) \leq \alpha_+(r)\,\sigma_K\,\sigma_{D(r)}$,
where $\alpha_+(r) = \max\{\Corr(K, -D(r)), 0\}$ (the non-negative part of
the correlation is the only part that can benefit).
Since $|D'| \leq L_D$ on $I$, $D$ is $L_D$-Lipschitz, so
$\sigma_{D(r)} \leq L_D\,\sigma_r$.
Therefore $-\Cov(K,D(r)) \leq \alpha_+(r)\,L_D\,\sigma_K\,\sigma_r$.

\textbf{Step 2: lower-bounding the Jensen penalty.}
Under $m_D$-strong convexity, for any $r$ with $\mathbb{E}[r]=\bar r$,
\begin{equation}
\mathbb{E}[D(r)] \;\geq\; D(\bar r) + \tfrac{m_D}{2}\,\mathbb{E}[(r-\bar r)^2]
= D(\bar r) + \tfrac{m_D}{2}\,\sigma_r^2,
\end{equation}
so $\mathbb{E}[D(r)] - D(\bar r) \geq \tfrac{m_D}{2}\sigma_r^2$ and
$-\mathbb{E}[K](\mathbb{E}[D(r)]-D(\bar r)) \leq -\tfrac{1}{2}\bar K\,m_D\,\sigma_r^2$.

\textbf{Step 3: combining.}
\begin{equation}
\label{eq:bound-corr}
\Delta_D(r) \;\leq\; \alpha_+(r)\,L_D\,\sigma_K\,\sigma_r
- \tfrac{1}{2}\,\bar K\,m_D\,\sigma_r^2.
\end{equation}

\begin{theoryrevision}
\textbf{Step 4: completing the square.}
For each fixed allocation, regard $\alpha_+(r)$ as a fixed coefficient.
The scalar quadratic in a variable $s\geq0$ is bounded by its value at
$s^\star=\alpha_+(r)L_D\sigma_K/(\bar K m_D)$.
Applying this algebraic inequality at $s=\sigma_r$ gives
\end{theoryrevision}
\begin{equation}
\Delta_D(r) \;\leq\;
\frac{\alpha_+(r)^2\,L_D^2\,\sigma_K^2}{2\,\bar K\,m_D}
\;\leq\;
\frac{L_D^2\,\sigma_K^2}{2\,\bar K\,m_D},
\end{equation}
where the last step uses $\alpha_+(r) \leq 1$, giving \pareneqref{eq:bound-universal}.
\end{proof}

\subsection[Proof of uniform optimality]{Proof of \Cref{thm:trap-global} (uniform optimality at a common loss minimum)}
\label{app:proof-trap-global}

\begin{theoryrevision}
Let $r^\star$ be a global minimiser of $\LL$ on $I$. For any feasible
allocation, $\mathcal J_{\LL}(r)\geq\bar K\LL(r^\star)$, and hence
\[
\Delta_{\LL}(r)=\bar K\LL(\bar r)-\mathcal J_{\LL}(r)
\leq\bar K[\LL(\bar r)-\LL(r^\star)].
\]
This proves \Cref{eq:remaining-headroom}. Now suppose $\bar r=r^\star$.
For every time step $t$,
\end{theoryrevision}
\begin{equation}
\LL(r_t) \;\geq\; \LL(r^{\star}) \quad \forall\, r_t,
\end{equation}
since $r^{\star}$ is the global minimiser.  Since $K_t \geq 0$,
\begin{equation}
\mathcal{J}_{\LL}(r) = \mathbb{E}[K\,\LL(r)]
\;\geq\; \mathbb{E}[K\,\LL(r^{\star})]
= \bar K\,\LL(r^{\star})
= \bar K\,\LL(\bar r)
= \mathcal{J}_{\LL}^{\mathrm{unif}}.
\end{equation}

For the equality condition: $\mathcal{J}_{\LL}(r) = \mathcal{J}_{\LL}^{\mathrm{unif}}$
requires $K_t(\LL(r_t) - \LL(r^{\star})) = 0$ for every $t$.  When $K_t > 0$
and $\LL(r_t) \geq \LL(r^{\star})$, this forces $\LL(r_t) = \LL(r^{\star})$.
If $r^{\star}$ is the unique minimiser of $\LL$ (ensured by \Cref{ass:unique}),
then $r_t = r^{\star}$ for all $t$.

\begin{theoryrevision}
\subsection[Local expansion of the shared loss]{Proof of \Cref{prop:trap} (local expansion of the shared loss)}
\label{app:proof-trap}
\end{theoryrevision}

\begin{theoryrevision}
\Cref{thm:trap-global} is the global statement. Expanding $\LL$ around $\Ebar$ gives the corresponding local terms.
\end{theoryrevision}

\begin{theoryrevision}
\begin{proposition}[Local expansion of the shared loss]
\label[proposition]{prop:trap}
Let $\delta_t=r_t-\Ebar$ and assume $\LL\in C^3$ on a neighborhood of
each segment between $\Ebar$ and $r_t$. Then
\begin{equation}
\label{eq:trap}
\Delta_{\LL}(r)
\;=\;
\underbrace{-\LL'(\Ebar)\,\Cov(K,r)}_{\text{first-order alignment}}
\;-\;
\underbrace{\tfrac{1}{2}\,\LL''(\Ebar)\,\mathbb{E}[K\delta^{2}]}_{\text{weighted curvature term}}
\;-\;\mathbb{E}[K\,R_{\LL,3}(\delta)].
\end{equation}
The first-order term vanishes when $\LL'(\Ebar)=0$ exactly. If
$\LL'(\Ebar)=0$ and $|\LL^{(3)}|\leq M_3$ on those segments, then
\begin{equation}
\label{eq:stationary-error}
\left|\Delta_{\LL}(r)+\tfrac12\LL''(\Ebar)\mathbb E[K\delta^2]\right|
\leq\tfrac{M_3}{6}\mathbb E[K|\delta|^3].
\end{equation}
When $\LL'(\Ebar)$ is only small, its term must be retained and can yield
positive gain. The exact near-minimum restriction is
\Cref{eq:remaining-headroom}, not nonpositivity throughout a neighborhood.
\end{proposition}
\end{theoryrevision}

\begin{theoryrevision}
At a stationary scalar operating point, $\LL'(\Ebar)=0$. Away from that
point, the first-order term remains part of the gain, even when its
coefficient is small.
\end{theoryrevision}

\begin{theoryrevision}
Apply \Cref{prop:taylor} with $D$ replaced by $\LL$. With
$R_{\LL,3}(\delta):=\LL(\bar r+\delta)-\LL(\bar r)-\LL'(\bar r)\delta
-\tfrac12\LL''(\bar r)\delta^2$, the exact expansion is
\[
\Delta_{\LL}(r)
=-\LL'(\bar r)\Cov(K,r)
-\tfrac12\LL''(\bar r)\mathbb{E}[K\delta^2]
-\mathbb{E}[K R_{\LL,3}(\delta)],
\]
which is \pareneqref{eq:trap}. At exact stationarity $\LL'(\bar r)=0$, the
first term vanishes. Applying the weighted remainder bound from
\Cref{prop:taylor} yields \Cref{eq:stationary-error}.
\end{theoryrevision}

\begin{theoryrevision}
\subsection{Conditional amplification of the curvature term}
\label{app:proof-locell}
\end{theoryrevision}

An optional construction is a base loss plus an auxiliary scalar cost:
\begin{equation}
\label{eq:representation-aware}
\LL(r) \;:=\; D(r) + \lambda\,\Daux(r),
\qquad \lambda \geq 0,
\end{equation}
Here $\Daux:I\to\mathbb R$ is an auxiliary cost. The weights multiply
the entire loss $\LL$ in $\mathbb E[K\LL(r)]$. This construction does not
enter the proof of uniform optimality, which only uses the common minimum.

Apply the exact Taylor expansion of \Cref{prop:taylor} with $\LL$ in place of $D$:
\begin{equation}
\Delta_{\LL}(r)
= -\LL'(\bar r)\Cov(K,r)
- \tfrac{1}{2}\,\LL''(\bar r)\,\mathbb{E}[K\delta^2]
- \mathbb{E}[K R_{\LL,3}(\delta)].
\end{equation}
\begin{theoryrevision}
This is valid under the segment smoothness assumptions of \Cref{prop:trap}.
At $\LL'(\bar r)=0$, the quadratic contribution is
$-\tfrac12\LL''(\bar r)\mathbb E[K\delta^2]$, with the weighted cubic
remainder bounded as in \Cref{eq:stationary-error}. When $\bar r=r^\star$
is the common global minimiser, exact nonimprovement follows separately
from \Cref{thm:trap-global}.
\end{theoryrevision}

\begin{theoryrevision}
Finally, if $D$ and $\Daux$ are $C^2$ near $\bar r$, the relation
$\LL(r)=D(r)+\lambda\Daux(r)$ implies
\end{theoryrevision}
\begin{equation}
\LL''(\bar r) = D''(\bar r) + \lambda\,\Daux''(\bar r).
\end{equation}
\begin{theoryrevision}
When $\lambda>0$ and the auxiliary cost has
$\Daux''(\bar r)>0$, the weighted quadratic term has strictly
larger curvature than the base loss for the same rate perturbation. This compares the quadratic curvature terms, not the complete
gains, whose first derivatives and remainders can also differ.
\end{theoryrevision}

\subsection{Lean formal verification}
\label{app:lean}

\framingrevision{All theoretical results in the paper, comprising the 11 named
mathematical results and the finite patch-budget identity, are formally verified}
in Lean~4.19.0 with mathlib~v4.19.0. The formalization
covers scalar invariance and its coupling corollary, nonimprovement under
independence, the exact improvement criterion, local improving allocations,
the quadratic-surrogate maximum, the global gain bound, KKT conditions,
uniform optimality and remaining absolute gain at a common loss minimum,
and the Taylor expansions with cubic remainder bounds.
The complete Lean project is provided as supplementary material.
An axiom audit checked all
111 public theorems and supporting lemmas; their only axiom dependencies
were \texttt{propext}, \texttt{Classical.choice}, and \texttt{Quot.sound}.
No proof uses \texttt{sorry}, \texttt{admit}, or a custom axiom.

\section{Significance of the validation-selected comparison}
\label{app:stats}

\paragraph{Estimands and variability.}
For each method, dataset, and horizon, both candidates are fixed using
three-seed validation results before five fresh paired runs are evaluated.
Tables report sample means and standard deviations over those five runs.
The MSE gain is $100(D-U)/D$, where $D$ and $U$ are arm means; it is not the
mean of per-seed percentage changes. A dataset-level point averages each
arm's four horizon means equally before computing the gain or time ratio.
Method summaries average the eight resulting relative effects equally.
Absolute MSE and MAE are never averaged across different datasets into a
single benchmark score.

\paragraph{Setting-level paired tests.}
We test the five paired raw MSE differences within each setting by
exhaustively enumerating all $2^5=32$ sign assignments to their mean.
The two-sided $p$-value is the fraction whose absolute mean is at least the
observed absolute mean. Benjamini--Hochberg correction is applied jointly
to all 96 tests. The smallest raw $p$ is $0.0625$ and the smallest adjusted
$q$ is $0.2727$; none reaches $0.05$. A percentile interval that excludes
zero is not interpreted as exact-test significance.

\paragraph{Dataset-level analysis.}
Horizons from the same dataset are dependent, so the main directional
check uses the eight dataset-level gains rather than treating 32 horizons
as independent observations. We report a two-sided exact signed-rank test
for each method. The test
uses the symmetric, independent dataset-effect null. Related ETT benchmarks
may remain dependent, and this suite is a fixed collection, so
population-level interpretation remains limited.
For a descriptive interval around the median dataset gain, we resample the
eight complete dataset effects with replacement 20,000 times. This is the
median of dataset-level gains, not the median of 32 horizon-level gains.
The resulting summaries are reported in \Cref{tab:stats-headline}.

\begin{table}[t]
\centering

\setlength{\tabcolsep}{4pt}\renewcommand{\arraystretch}{1.10}
\caption{\textbf{Dataset-level uncertainty and directional tests.} Each method contributes eight dataset-level MSE gains, defined as in \Cref{tab:main}. The median interval resamples the eight datasets 20,000 times. $p$ is the raw value from a two-sided exact signed-rank test for each method. Wins count strictly lower uniform seed-mean MSE across 32 horizon settings and eight dataset averages; they are descriptive.}\label{tab:stats-headline}
\begin{tabular}{@{}lrrrrr@{}}
\toprule
Method & Setting wins & Dataset wins & Median (\%) & 95\% interval & $p$ \\
\midrule
EntroPE & 18/32 & 6/8 & +0.22 & $[-0.67,+1.00]$ & 0.461 \\
HDMixer & 23/32 & 5/8 & +0.71 & $[-0.53,+2.02]$ & 0.195 \\
TimeMosaic & 12/32 & 2/8 & -0.16 & $[-0.66,+0.03]$ & 0.250 \\
\bottomrule
\end{tabular}
\end{table}

\paragraph{Interpretation.}
All three median intervals include zero and all method-wise raw $p$-values
exceed $0.05$. The observed means, win counts, and efficiency ratios are
reported as descriptive results.

\section{Complete validation-selected accuracy and efficiency results}
\label{app:full-tables}

The validation-selected adaptive and uniform configurations for all eight
datasets and four horizons are reported for EntroPE in
\Cref{tab:patch-sweep-entrope,tab:efficiency-entrope}, for HDMixer in
\Cref{tab:patch-sweep-hdmixer,tab:efficiency-hdmixer}, and for TimeMosaic in
\Cref{tab:patch-sweep-timemosaic,tab:efficiency-timemosaic}, with five final
runs per arm. Candidate identifiers map to the \framingrevision{candidate grids} in
\Cref{app:hyper}. Selection is performed independently for every
method--dataset--horizon--arm combination. No test-best candidate is chosen.
Mean test MSE and MAE are accompanied by across-run sample standard
deviations; boldface marks the lower unrounded mean, not significance.
The accompanying efficiency table uses exactly the same configurations.

\clearpage
\begingroup

\setlength{\tabcolsep}{1.75pt}\renewcommand{\arraystretch}{1.00}
\setlength{\LTleft}{\fill}\setlength{\LTright}{\fill}
\setlength{\LTpre}{6pt}\setlength{\LTpost}{6pt}
\setlength{\LTcapwidth}{\linewidth}
\begin{longtable}{@{}lrrrrrrrr@{}}
\caption{\textbf{EntroPE: selected configurations and test accuracy.} Both arms are selected by mean validation MSE over three tuning seeds and evaluated on five fresh seeds. Entries are final-run mean $\pm$ sample standard deviation. $c_D$ identifies the adaptive candidate in \Cref{app:hyper}; $p_U$ is the selected uniform length. Bold indicates the smaller unrounded mean within each metric pair; $G=100(D-U)/D$.}\label{tab:patch-sweep-entrope}\\
\toprule
 &  &  &  & \multicolumn{2}{c}{Test MSE $\downarrow$} & \multicolumn{2}{c}{Test MAE $\downarrow$} &  \\
\cmidrule(lr){5-6}\cmidrule(lr){7-8}
Dataset & $H$ & $c_D$ & $p_U$ & Adaptive & Uniform & Adaptive & Uniform & $G$ (\%) \\
\midrule
\endfirsthead
\toprule
 &  &  &  & \multicolumn{2}{c}{Test MSE $\downarrow$} & \multicolumn{2}{c}{Test MAE $\downarrow$} &  \\
\cmidrule(lr){5-6}\cmidrule(lr){7-8}
Dataset & $H$ & $c_D$ & $p_U$ & Adaptive & Uniform & Adaptive & Uniform & $G$ (\%) \\
\midrule
\endhead
ETTh1 & 96 & E2 & 24 & $0.3844\,{\scriptstyle\pm\,0.0025}$ & $\mathbf{0.3810}\,{\scriptstyle\pm\,0.0022}$ & $0.4008\,{\scriptstyle\pm\,0.0009}$ & $\mathbf{0.3989}\,{\scriptstyle\pm\,0.0010}$ & $\textcolor{bestblue}{+0.89}$ \\
 & 192 & E11 & 32 & $\mathbf{0.4269}\,{\scriptstyle\pm\,0.0041}$ & $0.4321\,{\scriptstyle\pm\,0.0068}$ & $\mathbf{0.4257}\,{\scriptstyle\pm\,0.0012}$ & $0.4267\,{\scriptstyle\pm\,0.0019}$ & $\textcolor{secondorange}{-1.23}$ \\
 & 336 & E8 & 32 & $0.4752\,{\scriptstyle\pm\,0.0075}$ & $\mathbf{0.4650}\,{\scriptstyle\pm\,0.0031}$ & $0.4482\,{\scriptstyle\pm\,0.0032}$ & $\mathbf{0.4438}\,{\scriptstyle\pm\,0.0013}$ & $\textcolor{bestblue}{+2.14}$ \\
 & 720 & E11 & 4 & $0.4890\,{\scriptstyle\pm\,0.0106}$ & $\mathbf{0.4796}\,{\scriptstyle\pm\,0.0153}$ & $0.4738\,{\scriptstyle\pm\,0.0073}$ & $\mathbf{0.4685}\,{\scriptstyle\pm\,0.0080}$ & $\textcolor{bestblue}{+1.92}$ \\
\addlinespace[2pt]
\rowcolor{tablerow}
ETTh2 & 96 & E1 & 8 & $\mathbf{0.3232}\,{\scriptstyle\pm\,0.0071}$ & $0.3233\,{\scriptstyle\pm\,0.0079}$ & $0.3654\,{\scriptstyle\pm\,0.0050}$ & $\mathbf{0.3642}\,{\scriptstyle\pm\,0.0050}$ & $\textcolor{secondorange}{-0.04}$ \\
\rowcolor{tablerow}
 & 192 & E6 & 8 & $0.4018\,{\scriptstyle\pm\,0.0133}$ & $\mathbf{0.3989}\,{\scriptstyle\pm\,0.0047}$ & $\mathbf{0.4123}\,{\scriptstyle\pm\,0.0083}$ & $0.4136\,{\scriptstyle\pm\,0.0044}$ & $\textcolor{bestblue}{+0.72}$ \\
\rowcolor{tablerow}
 & 336 & E0 & 8 & $\mathbf{0.4281}\,{\scriptstyle\pm\,0.0074}$ & $0.4323\,{\scriptstyle\pm\,0.0062}$ & $\mathbf{0.4417}\,{\scriptstyle\pm\,0.0052}$ & $0.4427\,{\scriptstyle\pm\,0.0039}$ & $\textcolor{secondorange}{-0.99}$ \\
\rowcolor{tablerow}
 & 720 & E4 & 32 & $0.4606\,{\scriptstyle\pm\,0.0068}$ & $\mathbf{0.4537}\,{\scriptstyle\pm\,0.0113}$ & $0.4668\,{\scriptstyle\pm\,0.0070}$ & $\mathbf{0.4616}\,{\scriptstyle\pm\,0.0074}$ & $\textcolor{bestblue}{+1.50}$ \\
\addlinespace[2pt]
ETTm1 & 96 & E10 & 24 & $\mathbf{0.3282}\,{\scriptstyle\pm\,0.0056}$ & $0.3316\,{\scriptstyle\pm\,0.0078}$ & $\mathbf{0.3630}\,{\scriptstyle\pm\,0.0054}$ & $0.3637\,{\scriptstyle\pm\,0.0049}$ & $\textcolor{secondorange}{-1.03}$ \\
 & 192 & E5 & 16 & $\mathbf{0.3715}\,{\scriptstyle\pm\,0.0048}$ & $0.3716\,{\scriptstyle\pm\,0.0052}$ & $0.3877\,{\scriptstyle\pm\,0.0026}$ & $\mathbf{0.3866}\,{\scriptstyle\pm\,0.0033}$ & $\textcolor{secondorange}{-0.02}$ \\
 & 336 & E9 & 16 & $\mathbf{0.4020}\,{\scriptstyle\pm\,0.0050}$ & $0.4050\,{\scriptstyle\pm\,0.0056}$ & $\mathbf{0.4064}\,{\scriptstyle\pm\,0.0041}$ & $0.4099\,{\scriptstyle\pm\,0.0021}$ & $\textcolor{secondorange}{-0.76}$ \\
 & 720 & E1 & 4 & $\mathbf{0.4572}\,{\scriptstyle\pm\,0.0054}$ & $0.4611\,{\scriptstyle\pm\,0.0036}$ & $\mathbf{0.4417}\,{\scriptstyle\pm\,0.0052}$ & $0.4444\,{\scriptstyle\pm\,0.0021}$ & $\textcolor{secondorange}{-0.85}$ \\
\addlinespace[2pt]
\rowcolor{tablerow}
ETTm2 & 96 & E7 & 24 & $\mathbf{0.1840}\,{\scriptstyle\pm\,0.0015}$ & $0.1885\,{\scriptstyle\pm\,0.0034}$ & $\mathbf{0.2682}\,{\scriptstyle\pm\,0.0014}$ & $0.2706\,{\scriptstyle\pm\,0.0026}$ & $\textcolor{secondorange}{-2.41}$ \\
\rowcolor{tablerow}
 & 192 & E1 & 4 & $\mathbf{0.2531}\,{\scriptstyle\pm\,0.0029}$ & $0.2539\,{\scriptstyle\pm\,0.0037}$ & $\mathbf{0.3122}\,{\scriptstyle\pm\,0.0011}$ & $0.3137\,{\scriptstyle\pm\,0.0025}$ & $\textcolor{secondorange}{-0.29}$ \\
\rowcolor{tablerow}
 & 336 & E3 & 32 & $\mathbf{0.3161}\,{\scriptstyle\pm\,0.0034}$ & $0.3320\,{\scriptstyle\pm\,0.0053}$ & $\mathbf{0.3530}\,{\scriptstyle\pm\,0.0029}$ & $0.3604\,{\scriptstyle\pm\,0.0024}$ & $\textcolor{secondorange}{-5.02}$ \\
\rowcolor{tablerow}
 & 720 & E6 & 4 & $0.4261\,{\scriptstyle\pm\,0.0046}$ & $\mathbf{0.4222}\,{\scriptstyle\pm\,0.0054}$ & $0.4184\,{\scriptstyle\pm\,0.0018}$ & $\mathbf{0.4164}\,{\scriptstyle\pm\,0.0047}$ & $\textcolor{bestblue}{+0.92}$ \\
\addlinespace[2pt]
Electricity & 96 & E11 & 4 & $0.1679\,{\scriptstyle\pm\,0.0015}$ & $\mathbf{0.1661}\,{\scriptstyle\pm\,0.0022}$ & $0.2638\,{\scriptstyle\pm\,0.0031}$ & $\mathbf{0.2608}\,{\scriptstyle\pm\,0.0046}$ & $\textcolor{bestblue}{+1.06}$ \\
 & 192 & E8 & 16 & $0.1786\,{\scriptstyle\pm\,0.0015}$ & $\mathbf{0.1770}\,{\scriptstyle\pm\,0.0010}$ & $0.2735\,{\scriptstyle\pm\,0.0031}$ & $\mathbf{0.2715}\,{\scriptstyle\pm\,0.0013}$ & $\textcolor{bestblue}{+0.88}$ \\
 & 336 & E10 & 16 & $0.1959\,{\scriptstyle\pm\,0.0017}$ & $\mathbf{0.1951}\,{\scriptstyle\pm\,0.0022}$ & $0.2906\,{\scriptstyle\pm\,0.0019}$ & $\mathbf{0.2894}\,{\scriptstyle\pm\,0.0032}$ & $\textcolor{bestblue}{+0.44}$ \\
 & 720 & E8 & 16 & $\mathbf{0.2337}\,{\scriptstyle\pm\,0.0007}$ & $0.2367\,{\scriptstyle\pm\,0.0035}$ & $\mathbf{0.3205}\,{\scriptstyle\pm\,0.0016}$ & $0.3238\,{\scriptstyle\pm\,0.0044}$ & $\textcolor{secondorange}{-1.30}$ \\
\addlinespace[2pt]
\rowcolor{tablerow}
Exchange & 96 & E0 & 4 & $0.0901\,{\scriptstyle\pm\,0.0030}$ & $\mathbf{0.0899}\,{\scriptstyle\pm\,0.0028}$ & $0.2107\,{\scriptstyle\pm\,0.0034}$ & $\mathbf{0.2106}\,{\scriptstyle\pm\,0.0028}$ & $\textcolor{bestblue}{+0.20}$ \\
\rowcolor{tablerow}
 & 192 & E11 & 8 & $\mathbf{0.1896}\,{\scriptstyle\pm\,0.0059}$ & $0.1913\,{\scriptstyle\pm\,0.0042}$ & $\mathbf{0.3104}\,{\scriptstyle\pm\,0.0048}$ & $0.3119\,{\scriptstyle\pm\,0.0037}$ & $\textcolor{secondorange}{-0.89}$ \\
\rowcolor{tablerow}
 & 336 & E8 & 32 & $0.3721\,{\scriptstyle\pm\,0.0129}$ & $\mathbf{0.3549}\,{\scriptstyle\pm\,0.0083}$ & $0.4405\,{\scriptstyle\pm\,0.0083}$ & $\mathbf{0.4303}\,{\scriptstyle\pm\,0.0046}$ & $\textcolor{bestblue}{+4.63}$ \\
\rowcolor{tablerow}
 & 720 & E10 & 8 & $\mathbf{0.9526}\,{\scriptstyle\pm\,0.0614}$ & $0.9638\,{\scriptstyle\pm\,0.0621}$ & $\mathbf{0.7361}\,{\scriptstyle\pm\,0.0235}$ & $0.7400\,{\scriptstyle\pm\,0.0253}$ & $\textcolor{secondorange}{-1.18}$ \\
\addlinespace[2pt]
Traffic & 96 & E11 & 4 & $0.5364\,{\scriptstyle\pm\,0.0056}$ & $\mathbf{0.5162}\,{\scriptstyle\pm\,0.0066}$ & $0.3214\,{\scriptstyle\pm\,0.0034}$ & $\mathbf{0.3121}\,{\scriptstyle\pm\,0.0046}$ & $\textcolor{bestblue}{+3.78}$ \\
 & 192 & E9 & 4 & $0.5302\,{\scriptstyle\pm\,0.0124}$ & $\mathbf{0.5263}\,{\scriptstyle\pm\,0.0129}$ & $0.3220\,{\scriptstyle\pm\,0.0033}$ & $\mathbf{0.3141}\,{\scriptstyle\pm\,0.0031}$ & $\textcolor{bestblue}{+0.73}$ \\
 & 336 & E10 & 8 & $0.5448\,{\scriptstyle\pm\,0.0110}$ & $\mathbf{0.5409}\,{\scriptstyle\pm\,0.0044}$ & $0.3330\,{\scriptstyle\pm\,0.0041}$ & $\mathbf{0.3228}\,{\scriptstyle\pm\,0.0057}$ & $\textcolor{bestblue}{+0.71}$ \\
 & 720 & E11 & 4 & $0.5904\,{\scriptstyle\pm\,0.0065}$ & $\mathbf{0.5777}\,{\scriptstyle\pm\,0.0093}$ & $0.3455\,{\scriptstyle\pm\,0.0049}$ & $\mathbf{0.3342}\,{\scriptstyle\pm\,0.0028}$ & $\textcolor{bestblue}{+2.16}$ \\
\addlinespace[2pt]
\rowcolor{tablerow}
Weather & 96 & E3 & 4 & $0.1707\,{\scriptstyle\pm\,0.0012}$ & $\mathbf{0.1697}\,{\scriptstyle\pm\,0.0009}$ & $0.2174\,{\scriptstyle\pm\,0.0014}$ & $\mathbf{0.2155}\,{\scriptstyle\pm\,0.0006}$ & $\textcolor{bestblue}{+0.63}$ \\
\rowcolor{tablerow}
 & 192 & E1 & 4 & $\mathbf{0.2153}\,{\scriptstyle\pm\,0.0009}$ & $0.2166\,{\scriptstyle\pm\,0.0011}$ & $\mathbf{0.2562}\,{\scriptstyle\pm\,0.0007}$ & $0.2570\,{\scriptstyle\pm\,0.0012}$ & $\textcolor{secondorange}{-0.59}$ \\
\rowcolor{tablerow}
 & 336 & E8 & 16 & $0.2708\,{\scriptstyle\pm\,0.0017}$ & $\mathbf{0.2702}\,{\scriptstyle\pm\,0.0009}$ & $\mathbf{0.2953}\,{\scriptstyle\pm\,0.0013}$ & $0.2958\,{\scriptstyle\pm\,0.0007}$ & $\textcolor{bestblue}{+0.19}$ \\
\rowcolor{tablerow}
 & 720 & E3 & 32 & $0.3433\,{\scriptstyle\pm\,0.0014}$ & $\mathbf{0.3429}\,{\scriptstyle\pm\,0.0005}$ & $\mathbf{0.3435}\,{\scriptstyle\pm\,0.0011}$ & $0.3440\,{\scriptstyle\pm\,0.0005}$ & $\textcolor{bestblue}{+0.13}$ \\
\bottomrule
\end{longtable}
\endgroup

\clearpage
\begingroup

\setlength{\tabcolsep}{4pt}\renewcommand{\arraystretch}{1.00}
\setlength{\LTleft}{\fill}\setlength{\LTright}{\fill}
\setlength{\LTpre}{6pt}\setlength{\LTpost}{6pt}
\begin{longtable}{@{}lrrrrrrr@{}}
\caption{\textbf{EntroPE: training and inference costs for the same selected configurations.} Each value averages five runs. Training is measured from steady-state timed epochs after optimizer warm-up; inference is median steady-state full-batch latency. $S=D/U$; values above one favor uniform.}\label{tab:efficiency-entrope}\\
\toprule
 &  & \multicolumn{3}{c}{Training (s/epoch) $\downarrow$} & \multicolumn{3}{c}{Inference (ms) $\downarrow$} \\
\cmidrule(lr){3-5}\cmidrule(lr){6-8}
Dataset & $H$ & Adaptive & Uniform & $S_{\mathrm{tr}}$ & Adaptive & Uniform & $S_{\mathrm{inf}}$ \\
\midrule
\endfirsthead
\toprule
 &  & \multicolumn{3}{c}{Training (s/epoch) $\downarrow$} & \multicolumn{3}{c}{Inference (ms) $\downarrow$} \\
\cmidrule(lr){3-5}\cmidrule(lr){6-8}
Dataset & $H$ & Adaptive & Uniform & $S_{\mathrm{tr}}$ & Adaptive & Uniform & $S_{\mathrm{inf}}$ \\
\midrule
\endhead
ETTh1 & 96 & 2.01 & 1.42 & 1.42 & 7.04 & 3.44 & 2.05 \\
 & 192 & 1.75 & 1.39 & 1.26 & 5.87 & 3.48 & 1.69 \\
 & 336 & 1.82 & 1.38 & 1.32 & 6.29 & 3.48 & 1.81 \\
 & 720 & 1.68 & 1.42 & 1.18 & 5.86 & 3.59 & 1.63 \\
\addlinespace[2pt]
\rowcolor{tablerow}
ETTh2 & 96 & 2.25 & 1.58 & 1.42 & 7.63 & 3.77 & 2.02 \\
\rowcolor{tablerow}
 & 192 & 2.06 & 1.54 & 1.33 & 6.91 & 3.76 & 1.84 \\
\rowcolor{tablerow}
 & 336 & 2.22 & 1.52 & 1.46 & 7.68 & 3.77 & 2.04 \\
\rowcolor{tablerow}
 & 720 & 2.05 & 1.41 & 1.45 & 7.32 & 3.71 & 1.97 \\
\addlinespace[2pt]
ETTm1 & 96 & 10.53 & 7.88 & 1.34 & 4.66 & 2.49 & 1.87 \\
 & 192 & 10.80 & 7.92 & 1.36 & 4.99 & 2.51 & 1.99 \\
 & 336 & 10.46 & 7.85 & 1.33 & 4.74 & 2.51 & 1.89 \\
 & 720 & 11.24 & 8.09 & 1.39 & 5.22 & 2.60 & 2.01 \\
\addlinespace[2pt]
\rowcolor{tablerow}
ETTm2 & 96 & 11.11 & 8.08 & 1.37 & 5.04 & 2.59 & 1.94 \\
\rowcolor{tablerow}
 & 192 & 11.73 & 8.44 & 1.39 & 5.39 & 2.69 & 2.00 \\
\rowcolor{tablerow}
 & 336 & 11.71 & 8.11 & 1.44 & 5.31 & 2.63 & 2.02 \\
\rowcolor{tablerow}
 & 720 & 11.21 & 8.49 & 1.32 & 5.15 & 2.72 & 1.89 \\
\addlinespace[2pt]
Electricity & 96 & 158.02 & 141.11 & 1.12 & 102.82 & 84.66 & 1.21 \\
 & 192 & 170.56 & 123.48 & 1.38 & 115.83 & 74.11 & 1.56 \\
 & 336 & 159.07 & 123.46 & 1.29 & 104.72 & 74.33 & 1.41 \\
 & 720 & 169.60 & 123.44 & 1.37 & 116.64 & 74.91 & 1.56 \\
\addlinespace[2pt]
\rowcolor{tablerow}
Exchange & 96 & 1.93 & 1.45 & 1.33 & 4.80 & 2.06 & 2.33 \\
\rowcolor{tablerow}
 & 192 & 1.84 & 1.40 & 1.31 & 4.28 & 2.07 & 2.06 \\
\rowcolor{tablerow}
 & 336 & 1.74 & 1.34 & 1.29 & 4.43 & 2.04 & 2.17 \\
\rowcolor{tablerow}
 & 720 & 1.63 & 1.30 & 1.26 & 4.30 & 2.06 & 2.08 \\
\addlinespace[2pt]
Traffic & 96 & 110.41 & 61.25 & 1.80 & 30.51 & 13.43 & 2.27 \\
 & 192 & 114.28 & 61.54 & 1.86 & 32.59 & 13.45 & 2.42 \\
 & 336 & 111.62 & 57.59 & 1.94 & 31.17 & 12.40 & 2.51 \\
 & 720 & 110.44 & 63.60 & 1.74 & 30.61 & 13.65 & 2.24 \\
\addlinespace[2pt]
\rowcolor{tablerow}
Weather & 96 & 22.09 & 16.10 & 1.37 & 32.95 & 18.21 & 1.81 \\
\rowcolor{tablerow}
 & 192 & 22.43 & 16.06 & 1.40 & 33.92 & 18.24 & 1.86 \\
\rowcolor{tablerow}
 & 336 & 19.18 & 14.17 & 1.35 & 27.56 & 16.06 & 1.72 \\
\rowcolor{tablerow}
 & 720 & 22.26 & 14.09 & 1.58 & 33.18 & 15.77 & 2.10 \\
\bottomrule
\end{longtable}
\endgroup

\clearpage
\begingroup

\setlength{\tabcolsep}{1.75pt}\renewcommand{\arraystretch}{1.00}
\setlength{\LTleft}{\fill}\setlength{\LTright}{\fill}
\setlength{\LTpre}{6pt}\setlength{\LTpost}{6pt}
\setlength{\LTcapwidth}{\linewidth}
\begin{longtable}{@{}lrrrrrrrr@{}}
\caption{\textbf{HDMixer: selected configurations and test accuracy.} Both arms are selected by mean validation MSE over three tuning seeds and evaluated on five fresh seeds. Entries are final-run mean $\pm$ sample standard deviation. $c_D$ identifies the adaptive candidate in \Cref{app:hyper}; $p_U$ is the selected uniform length. Bold indicates the smaller unrounded mean within each metric pair; $G=100(D-U)/D$.}\label{tab:patch-sweep-hdmixer}\\
\toprule
 &  &  &  & \multicolumn{2}{c}{Test MSE $\downarrow$} & \multicolumn{2}{c}{Test MAE $\downarrow$} &  \\
\cmidrule(lr){5-6}\cmidrule(lr){7-8}
Dataset & $H$ & $c_D$ & $p_U$ & Adaptive & Uniform & Adaptive & Uniform & $G$ (\%) \\
\midrule
\endfirsthead
\toprule
 &  &  &  & \multicolumn{2}{c}{Test MSE $\downarrow$} & \multicolumn{2}{c}{Test MAE $\downarrow$} &  \\
\cmidrule(lr){5-6}\cmidrule(lr){7-8}
Dataset & $H$ & $c_D$ & $p_U$ & Adaptive & Uniform & Adaptive & Uniform & $G$ (\%) \\
\midrule
\endhead
ETTh1 & 96 & H1 & 4 & $0.3688\,{\scriptstyle\pm\,0.0007}$ & $\mathbf{0.3646}\,{\scriptstyle\pm\,0.0002}$ & $0.3946\,{\scriptstyle\pm\,0.0005}$ & $\mathbf{0.3917}\,{\scriptstyle\pm\,0.0005}$ & $\textcolor{bestblue}{+1.14}$ \\
 & 192 & H2 & 4 & $0.4088\,{\scriptstyle\pm\,0.0013}$ & $\mathbf{0.4023}\,{\scriptstyle\pm\,0.0005}$ & $0.4159\,{\scriptstyle\pm\,0.0009}$ & $\mathbf{0.4118}\,{\scriptstyle\pm\,0.0008}$ & $\textcolor{bestblue}{+1.60}$ \\
 & 336 & H3 & 32 & $0.3971\,{\scriptstyle\pm\,0.0017}$ & $\mathbf{0.3892}\,{\scriptstyle\pm\,0.0015}$ & $0.4197\,{\scriptstyle\pm\,0.0009}$ & $\mathbf{0.4123}\,{\scriptstyle\pm\,0.0024}$ & $\textcolor{bestblue}{+1.98}$ \\
 & 720 & H0 & 4 & $0.4577\,{\scriptstyle\pm\,0.0086}$ & $\mathbf{0.4442}\,{\scriptstyle\pm\,0.0019}$ & $0.4612\,{\scriptstyle\pm\,0.0049}$ & $\mathbf{0.4564}\,{\scriptstyle\pm\,0.0012}$ & $\textcolor{bestblue}{+2.95}$ \\
\addlinespace[2pt]
\rowcolor{tablerow}
ETTh2 & 96 & H0 & 8 & $0.2651\,{\scriptstyle\pm\,0.0020}$ & $\mathbf{0.2625}\,{\scriptstyle\pm\,0.0007}$ & $0.3314\,{\scriptstyle\pm\,0.0019}$ & $\mathbf{0.3286}\,{\scriptstyle\pm\,0.0009}$ & $\textcolor{bestblue}{+0.95}$ \\
\rowcolor{tablerow}
 & 192 & H0 & 8 & $0.3208\,{\scriptstyle\pm\,0.0026}$ & $\mathbf{0.3179}\,{\scriptstyle\pm\,0.0018}$ & $0.3723\,{\scriptstyle\pm\,0.0031}$ & $\mathbf{0.3686}\,{\scriptstyle\pm\,0.0025}$ & $\textcolor{bestblue}{+0.93}$ \\
\rowcolor{tablerow}
 & 336 & H0 & 8 & $0.3109\,{\scriptstyle\pm\,0.0041}$ & $\mathbf{0.3104}\,{\scriptstyle\pm\,0.0036}$ & $0.3724\,{\scriptstyle\pm\,0.0033}$ & $\mathbf{0.3709}\,{\scriptstyle\pm\,0.0028}$ & $\textcolor{bestblue}{+0.16}$ \\
\rowcolor{tablerow}
 & 720 & H0 & 8 & $0.4156\,{\scriptstyle\pm\,0.0196}$ & $\mathbf{0.3952}\,{\scriptstyle\pm\,0.0137}$ & $0.4434\,{\scriptstyle\pm\,0.0124}$ & $\mathbf{0.4308}\,{\scriptstyle\pm\,0.0091}$ & $\textcolor{bestblue}{+4.92}$ \\
\addlinespace[2pt]
ETTm1 & 96 & H3 & 16 & $\mathbf{0.2985}\,{\scriptstyle\pm\,0.0020}$ & $0.3020\,{\scriptstyle\pm\,0.0010}$ & $0.3501\,{\scriptstyle\pm\,0.0012}$ & $\mathbf{0.3464}\,{\scriptstyle\pm\,0.0009}$ & $\textcolor{secondorange}{-1.17}$ \\
 & 192 & H15 & 4 & $\mathbf{0.3360}\,{\scriptstyle\pm\,0.0011}$ & $0.3387\,{\scriptstyle\pm\,0.0020}$ & $0.3715\,{\scriptstyle\pm\,0.0010}$ & $\mathbf{0.3677}\,{\scriptstyle\pm\,0.0016}$ & $\textcolor{secondorange}{-0.81}$ \\
 & 336 & H2 & 4 & $\mathbf{0.3714}\,{\scriptstyle\pm\,0.0023}$ & $0.3725\,{\scriptstyle\pm\,0.0021}$ & $0.3926\,{\scriptstyle\pm\,0.0030}$ & $\mathbf{0.3877}\,{\scriptstyle\pm\,0.0017}$ & $\textcolor{secondorange}{-0.29}$ \\
 & 720 & H0 & 4 & $\mathbf{0.4305}\,{\scriptstyle\pm\,0.0016}$ & $0.4308\,{\scriptstyle\pm\,0.0022}$ & $0.4215\,{\scriptstyle\pm\,0.0016}$ & $\mathbf{0.4195}\,{\scriptstyle\pm\,0.0018}$ & $\textcolor{secondorange}{-0.08}$ \\
\addlinespace[2pt]
\rowcolor{tablerow}
ETTm2 & 96 & H0 & 4 & $0.1658\,{\scriptstyle\pm\,0.0015}$ & $\mathbf{0.1650}\,{\scriptstyle\pm\,0.0008}$ & $0.2551\,{\scriptstyle\pm\,0.0015}$ & $\mathbf{0.2540}\,{\scriptstyle\pm\,0.0005}$ & $\textcolor{bestblue}{+0.45}$ \\
\rowcolor{tablerow}
 & 192 & H12 & 8 & $0.2214\,{\scriptstyle\pm\,0.0010}$ & $\mathbf{0.2201}\,{\scriptstyle\pm\,0.0010}$ & $0.2929\,{\scriptstyle\pm\,0.0010}$ & $\mathbf{0.2915}\,{\scriptstyle\pm\,0.0009}$ & $\textcolor{bestblue}{+0.63}$ \\
\rowcolor{tablerow}
 & 336 & H12 & 8 & $\mathbf{0.2774}\,{\scriptstyle\pm\,0.0024}$ & $0.2777\,{\scriptstyle\pm\,0.0009}$ & $0.3298\,{\scriptstyle\pm\,0.0016}$ & $\mathbf{0.3292}\,{\scriptstyle\pm\,0.0014}$ & $\textcolor{secondorange}{-0.08}$ \\
\rowcolor{tablerow}
 & 720 & H1 & 4 & $\mathbf{0.3764}\,{\scriptstyle\pm\,0.0203}$ & $0.3840\,{\scriptstyle\pm\,0.0322}$ & $\mathbf{0.3895}\,{\scriptstyle\pm\,0.0090}$ & $0.3926\,{\scriptstyle\pm\,0.0138}$ & $\textcolor{secondorange}{-2.00}$ \\
\addlinespace[2pt]
Electricity & 96 & H4 & 16 & $0.1619\,{\scriptstyle\pm\,0.0104}$ & $\mathbf{0.1487}\,{\scriptstyle\pm\,0.0010}$ & $0.2605\,{\scriptstyle\pm\,0.0124}$ & $\mathbf{0.2503}\,{\scriptstyle\pm\,0.0013}$ & $\textcolor{bestblue}{+8.15}$ \\
 & 192 & H0 & 16 & $0.1798\,{\scriptstyle\pm\,0.0143}$ & $\mathbf{0.1667}\,{\scriptstyle\pm\,0.0014}$ & $0.2700\,{\scriptstyle\pm\,0.0080}$ & $\mathbf{0.2660}\,{\scriptstyle\pm\,0.0029}$ & $\textcolor{bestblue}{+7.29}$ \\
 & 336 & H0 & 8 & $0.1839\,{\scriptstyle\pm\,0.0051}$ & $\mathbf{0.1680}\,{\scriptstyle\pm\,0.0026}$ & $0.2822\,{\scriptstyle\pm\,0.0038}$ & $\mathbf{0.2662}\,{\scriptstyle\pm\,0.0028}$ & $\textcolor{bestblue}{+8.65}$ \\
 & 720 & H4 & 16 & $0.2404\,{\scriptstyle\pm\,0.0524}$ & $\mathbf{0.2040}\,{\scriptstyle\pm\,0.0009}$ & $0.3259\,{\scriptstyle\pm\,0.0374}$ & $\mathbf{0.2979}\,{\scriptstyle\pm\,0.0012}$ & $\textcolor{bestblue}{+15.15}$ \\
\addlinespace[2pt]
\rowcolor{tablerow}
Exchange & 96 & H13 & 32 & $0.0871\,{\scriptstyle\pm\,0.0011}$ & $\mathbf{0.0845}\,{\scriptstyle\pm\,0.0011}$ & $0.2063\,{\scriptstyle\pm\,0.0012}$ & $\mathbf{0.2009}\,{\scriptstyle\pm\,0.0013}$ & $\textcolor{bestblue}{+2.96}$ \\
\rowcolor{tablerow}
 & 192 & H5 & 32 & $0.1854\,{\scriptstyle\pm\,0.0040}$ & $\mathbf{0.1809}\,{\scriptstyle\pm\,0.0037}$ & $0.3057\,{\scriptstyle\pm\,0.0039}$ & $\mathbf{0.2998}\,{\scriptstyle\pm\,0.0031}$ & $\textcolor{bestblue}{+2.44}$ \\
\rowcolor{tablerow}
 & 336 & H4 & 32 & $0.3520\,{\scriptstyle\pm\,0.0111}$ & $\mathbf{0.3443}\,{\scriptstyle\pm\,0.0051}$ & $0.4303\,{\scriptstyle\pm\,0.0084}$ & $\mathbf{0.4220}\,{\scriptstyle\pm\,0.0039}$ & $\textcolor{bestblue}{+2.21}$ \\
\rowcolor{tablerow}
 & 720 & H9 & 24 & $1.1067\,{\scriptstyle\pm\,0.0149}$ & $\mathbf{1.1053}\,{\scriptstyle\pm\,0.0104}$ & $0.7955\,{\scriptstyle\pm\,0.0052}$ & $\mathbf{0.7951}\,{\scriptstyle\pm\,0.0040}$ & $\textcolor{bestblue}{+0.13}$ \\
\addlinespace[2pt]
Traffic & 96 & H9 & 24 & $\mathbf{0.4057}\,{\scriptstyle\pm\,0.0018}$ & $0.4093\,{\scriptstyle\pm\,0.0024}$ & $\mathbf{0.2821}\,{\scriptstyle\pm\,0.0017}$ & $0.2856\,{\scriptstyle\pm\,0.0048}$ & $\textcolor{secondorange}{-0.90}$ \\
 & 192 & H10 & 24 & $0.4254\,{\scriptstyle\pm\,0.0032}$ & $\mathbf{0.4220}\,{\scriptstyle\pm\,0.0002}$ & $0.3012\,{\scriptstyle\pm\,0.0069}$ & $\mathbf{0.2917}\,{\scriptstyle\pm\,0.0014}$ & $\textcolor{bestblue}{+0.80}$ \\
 & 336 & H9 & 32 & $0.4480\,{\scriptstyle\pm\,0.0052}$ & $\mathbf{0.4422}\,{\scriptstyle\pm\,0.0055}$ & $0.3230\,{\scriptstyle\pm\,0.0064}$ & $\mathbf{0.3140}\,{\scriptstyle\pm\,0.0081}$ & $\textcolor{bestblue}{+1.29}$ \\
 & 720 & H9 & 32 & $\mathbf{0.4661}\,{\scriptstyle\pm\,0.0035}$ & $0.4738\,{\scriptstyle\pm\,0.0037}$ & $\mathbf{0.3319}\,{\scriptstyle\pm\,0.0059}$ & $0.3373\,{\scriptstyle\pm\,0.0042}$ & $\textcolor{secondorange}{-1.67}$ \\
\addlinespace[2pt]
\rowcolor{tablerow}
Weather & 96 & H8 & 32 & $0.1545\,{\scriptstyle\pm\,0.0025}$ & $\mathbf{0.1515}\,{\scriptstyle\pm\,0.0014}$ & $0.2106\,{\scriptstyle\pm\,0.0028}$ & $\mathbf{0.2072}\,{\scriptstyle\pm\,0.0029}$ & $\textcolor{bestblue}{+1.94}$ \\
\rowcolor{tablerow}
 & 192 & H8 & 24 & $0.2004\,{\scriptstyle\pm\,0.0008}$ & $\mathbf{0.1991}\,{\scriptstyle\pm\,0.0018}$ & $0.2518\,{\scriptstyle\pm\,0.0016}$ & $\mathbf{0.2500}\,{\scriptstyle\pm\,0.0014}$ & $\textcolor{bestblue}{+0.64}$ \\
\rowcolor{tablerow}
 & 336 & H8 & 24 & $\mathbf{0.2544}\,{\scriptstyle\pm\,0.0032}$ & $0.2562\,{\scriptstyle\pm\,0.0071}$ & $\mathbf{0.2927}\,{\scriptstyle\pm\,0.0021}$ & $0.2928\,{\scriptstyle\pm\,0.0072}$ & $\textcolor{secondorange}{-0.74}$ \\
\rowcolor{tablerow}
 & 720 & H10 & 24 & $0.3336\,{\scriptstyle\pm\,0.0051}$ & $\mathbf{0.3315}\,{\scriptstyle\pm\,0.0155}$ & $0.3463\,{\scriptstyle\pm\,0.0027}$ & $\mathbf{0.3433}\,{\scriptstyle\pm\,0.0077}$ & $\textcolor{bestblue}{+0.65}$ \\
\bottomrule
\end{longtable}
\endgroup

\clearpage
\begingroup

\setlength{\tabcolsep}{4pt}\renewcommand{\arraystretch}{1.00}
\setlength{\LTleft}{\fill}\setlength{\LTright}{\fill}
\setlength{\LTpre}{6pt}\setlength{\LTpost}{6pt}
\begin{longtable}{@{}lrrrrrrr@{}}
\caption{\textbf{HDMixer: training and inference costs for the same selected configurations.} Each value averages five runs. Training is measured from steady-state timed epochs after optimizer warm-up; inference is median steady-state full-batch latency. $S=D/U$; values above one favor uniform.}\label{tab:efficiency-hdmixer}\\
\toprule
 &  & \multicolumn{3}{c}{Training (s/epoch) $\downarrow$} & \multicolumn{3}{c}{Inference (ms) $\downarrow$} \\
\cmidrule(lr){3-5}\cmidrule(lr){6-8}
Dataset & $H$ & Adaptive & Uniform & $S_{\mathrm{tr}}$ & Adaptive & Uniform & $S_{\mathrm{inf}}$ \\
\midrule
\endfirsthead
\toprule
 &  & \multicolumn{3}{c}{Training (s/epoch) $\downarrow$} & \multicolumn{3}{c}{Inference (ms) $\downarrow$} \\
\cmidrule(lr){3-5}\cmidrule(lr){6-8}
Dataset & $H$ & Adaptive & Uniform & $S_{\mathrm{tr}}$ & Adaptive & Uniform & $S_{\mathrm{inf}}$ \\
\midrule
\endhead
ETTh1 & 96 & 1.78 & 0.30 & 6.02 & 11.90 & 2.59 & 4.59 \\
 & 192 & 1.75 & 0.30 & 5.82 & 11.87 & 2.59 & 4.58 \\
 & 336 & 0.95 & 0.21 & 4.56 & 15.20 & 1.96 & 7.75 \\
 & 720 & 1.77 & 0.61 & 2.92 & 27.96 & 10.15 & 2.75 \\
\addlinespace[2pt]
\rowcolor{tablerow}
ETTh2 & 96 & 1.79 & 0.22 & 8.17 & 11.81 & 1.45 & 8.16 \\
\rowcolor{tablerow}
 & 192 & 1.61 & 0.18 & 8.82 & 23.52 & 2.73 & 8.61 \\
\rowcolor{tablerow}
 & 336 & 1.63 & 0.20 & 7.96 & 23.50 & 2.75 & 8.54 \\
\rowcolor{tablerow}
 & 720 & 1.54 & 0.26 & 6.01 & 23.40 & 2.80 & 8.35 \\
\addlinespace[2pt]
ETTm1 & 96 & 3.75 & 0.68 & 5.49 & 6.59 & 0.88 & 7.50 \\
 & 192 & 3.74 & 1.12 & 3.35 & 6.58 & 2.59 & 2.54 \\
 & 336 & 7.38 & 1.15 & 6.40 & 11.95 & 2.62 & 4.55 \\
 & 720 & 7.40 & 1.27 & 5.85 & 11.98 & 2.65 & 4.52 \\
\addlinespace[2pt]
\rowcolor{tablerow}
ETTm2 & 96 & 7.37 & 1.11 & 6.63 & 11.92 & 2.59 & 4.60 \\
\rowcolor{tablerow}
 & 192 & 3.75 & 0.81 & 4.63 & 6.58 & 1.45 & 4.54 \\
\rowcolor{tablerow}
 & 336 & 3.76 & 0.90 & 4.19 & 6.60 & 1.47 & 4.49 \\
\rowcolor{tablerow}
 & 720 & 7.42 & 1.26 & 5.87 & 11.98 & 2.65 & 4.52 \\
\addlinespace[2pt]
Electricity & 96 & 119.37 & 9.12 & 13.10 & 35.56 & 3.82 & 9.30 \\
 & 192 & 197.64 & 9.39 & 21.05 & 67.36 & 3.84 & 17.54 \\
 & 336 & 197.27 & 16.41 & 12.02 & 66.42 & 6.11 & 10.87 \\
 & 720 & 117.79 & 15.11 & 7.79 & 35.89 & 3.97 & 9.04 \\
\addlinespace[2pt]
\rowcolor{tablerow}
Exchange & 96 & 1.07 & 0.56 & 1.91 & 1.36 & 0.44 & 3.07 \\
\rowcolor{tablerow}
 & 192 & 1.04 & 0.55 & 1.90 & 1.36 & 0.44 & 3.12 \\
\rowcolor{tablerow}
 & 336 & 1.02 & 0.55 & 1.86 & 1.35 & 0.43 & 3.12 \\
\rowcolor{tablerow}
 & 720 & 0.75 & 0.52 & 1.45 & 1.03 & 0.42 & 2.47 \\
\addlinespace[2pt]
Traffic & 96 & 95.80 & 14.72 & 6.51 & 13.70 & 2.55 & 5.38 \\
 & 192 & 129.09 & 24.50 & 5.27 & 18.47 & 3.53 & 5.23 \\
 & 336 & 260.55 & 32.52 & 8.01 & 32.84 & 6.30 & 5.21 \\
 & 720 & 185.95 & 30.46 & 6.10 & 28.17 & 5.38 & 5.24 \\
\addlinespace[2pt]
\rowcolor{tablerow}
Weather & 96 & 9.79 & 4.15 & 2.36 & 1.98 & 0.44 & 4.47 \\
\rowcolor{tablerow}
 & 192 & 9.85 & 4.28 & 2.31 & 1.98 & 0.47 & 4.21 \\
\rowcolor{tablerow}
 & 336 & 10.14 & 4.43 & 2.29 & 2.00 & 0.47 & 4.22 \\
\rowcolor{tablerow}
 & 720 & 8.70 & 5.03 & 1.73 & 1.78 & 0.48 & 3.71 \\
\bottomrule
\end{longtable}
\endgroup

\clearpage
\begingroup

\setlength{\tabcolsep}{1.75pt}\renewcommand{\arraystretch}{1.00}
\setlength{\LTleft}{\fill}\setlength{\LTright}{\fill}
\setlength{\LTpre}{6pt}\setlength{\LTpost}{6pt}
\setlength{\LTcapwidth}{\linewidth}
\begin{longtable}{@{}lrrrrrrrr@{}}
\caption{\textbf{TimeMosaic: selected configurations and test accuracy.} Both arms are selected by mean validation MSE over three tuning seeds and evaluated on five fresh seeds. Entries are final-run mean $\pm$ sample standard deviation. $c_D$ identifies the adaptive candidate in \Cref{app:hyper}; $p_U$ is the selected uniform length. Bold indicates the smaller unrounded mean within each metric pair; $G=100(D-U)/D$.}\label{tab:patch-sweep-timemosaic}\\
\toprule
 &  &  &  & \multicolumn{2}{c}{Test MSE $\downarrow$} & \multicolumn{2}{c}{Test MAE $\downarrow$} &  \\
\cmidrule(lr){5-6}\cmidrule(lr){7-8}
Dataset & $H$ & $c_D$ & $p_U$ & Adaptive & Uniform & Adaptive & Uniform & $G$ (\%) \\
\midrule
\endfirsthead
\toprule
 &  &  &  & \multicolumn{2}{c}{Test MSE $\downarrow$} & \multicolumn{2}{c}{Test MAE $\downarrow$} &  \\
\cmidrule(lr){5-6}\cmidrule(lr){7-8}
Dataset & $H$ & $c_D$ & $p_U$ & Adaptive & Uniform & Adaptive & Uniform & $G$ (\%) \\
\midrule
\endhead
ETTh1 & 96 & T8 & 32 & $\mathbf{0.3611}\,{\scriptstyle\pm\,0.0061}$ & $0.3639\,{\scriptstyle\pm\,0.0067}$ & $\mathbf{0.3869}\,{\scriptstyle\pm\,0.0047}$ & $0.3886\,{\scriptstyle\pm\,0.0036}$ & $\textcolor{secondorange}{-0.79}$ \\
 & 192 & T6 & 16 & $0.4271\,{\scriptstyle\pm\,0.0150}$ & $\mathbf{0.4238}\,{\scriptstyle\pm\,0.0184}$ & $0.4307\,{\scriptstyle\pm\,0.0121}$ & $\mathbf{0.4269}\,{\scriptstyle\pm\,0.0146}$ & $\textcolor{bestblue}{+0.78}$ \\
 & 336 & T8 & 8 & $\mathbf{0.4494}\,{\scriptstyle\pm\,0.0181}$ & $0.4498\,{\scriptstyle\pm\,0.0225}$ & $0.4451\,{\scriptstyle\pm\,0.0106}$ & $\mathbf{0.4446}\,{\scriptstyle\pm\,0.0150}$ & $\textcolor{secondorange}{-0.09}$ \\
 & 720 & T8 & 8 & $\mathbf{0.4467}\,{\scriptstyle\pm\,0.0043}$ & $0.4578\,{\scriptstyle\pm\,0.0094}$ & $\mathbf{0.4600}\,{\scriptstyle\pm\,0.0028}$ & $0.4693\,{\scriptstyle\pm\,0.0047}$ & $\textcolor{secondorange}{-2.50}$ \\
\addlinespace[2pt]
\rowcolor{tablerow}
ETTh2 & 96 & T4 & 32 & $\mathbf{0.2947}\,{\scriptstyle\pm\,0.0054}$ & $0.2986\,{\scriptstyle\pm\,0.0051}$ & $0.3464\,{\scriptstyle\pm\,0.0048}$ & $\mathbf{0.3444}\,{\scriptstyle\pm\,0.0037}$ & $\textcolor{secondorange}{-1.35}$ \\
\rowcolor{tablerow}
 & 192 & T5 & 16 & $\mathbf{0.3451}\,{\scriptstyle\pm\,0.0089}$ & $0.3507\,{\scriptstyle\pm\,0.0037}$ & $\mathbf{0.3792}\,{\scriptstyle\pm\,0.0040}$ & $0.3831\,{\scriptstyle\pm\,0.0026}$ & $\textcolor{secondorange}{-1.61}$ \\
\rowcolor{tablerow}
 & 336 & T8 & 32 & $\mathbf{0.3809}\,{\scriptstyle\pm\,0.0064}$ & $0.3817\,{\scriptstyle\pm\,0.0113}$ & $0.4058\,{\scriptstyle\pm\,0.0034}$ & $\mathbf{0.4046}\,{\scriptstyle\pm\,0.0043}$ & $\textcolor{secondorange}{-0.22}$ \\
\rowcolor{tablerow}
 & 720 & T8 & 16 & $0.4217\,{\scriptstyle\pm\,0.0126}$ & $\mathbf{0.4133}\,{\scriptstyle\pm\,0.0141}$ & $0.4399\,{\scriptstyle\pm\,0.0077}$ & $\mathbf{0.4315}\,{\scriptstyle\pm\,0.0061}$ & $\textcolor{bestblue}{+1.99}$ \\
\addlinespace[2pt]
ETTm1 & 96 & T2 & 16 & $\mathbf{0.2880}\,{\scriptstyle\pm\,0.0015}$ & $0.2913\,{\scriptstyle\pm\,0.0017}$ & $\mathbf{0.3310}\,{\scriptstyle\pm\,0.0017}$ & $0.3319\,{\scriptstyle\pm\,0.0006}$ & $\textcolor{secondorange}{-1.15}$ \\
 & 192 & T5 & 16 & $\mathbf{0.3327}\,{\scriptstyle\pm\,0.0032}$ & $0.3336\,{\scriptstyle\pm\,0.0029}$ & $\mathbf{0.3593}\,{\scriptstyle\pm\,0.0016}$ & $0.3600\,{\scriptstyle\pm\,0.0010}$ & $\textcolor{secondorange}{-0.27}$ \\
 & 336 & T8 & 8 & $0.3662\,{\scriptstyle\pm\,0.0022}$ & $\mathbf{0.3656}\,{\scriptstyle\pm\,0.0048}$ & $\mathbf{0.3795}\,{\scriptstyle\pm\,0.0006}$ & $0.3800\,{\scriptstyle\pm\,0.0009}$ & $\textcolor{bestblue}{+0.18}$ \\
 & 720 & T6 & 8 & $\mathbf{0.4340}\,{\scriptstyle\pm\,0.0120}$ & $0.4367\,{\scriptstyle\pm\,0.0077}$ & $\mathbf{0.4195}\,{\scriptstyle\pm\,0.0031}$ & $0.4221\,{\scriptstyle\pm\,0.0021}$ & $\textcolor{secondorange}{-0.62}$ \\
\addlinespace[2pt]
\rowcolor{tablerow}
ETTm2 & 96 & T4 & 8 & $\mathbf{0.1702}\,{\scriptstyle\pm\,0.0039}$ & $0.1731\,{\scriptstyle\pm\,0.0050}$ & $\mathbf{0.2533}\,{\scriptstyle\pm\,0.0036}$ & $0.2550\,{\scriptstyle\pm\,0.0034}$ & $\textcolor{secondorange}{-1.66}$ \\
\rowcolor{tablerow}
 & 192 & T5 & 16 & $\mathbf{0.2295}\,{\scriptstyle\pm\,0.0056}$ & $0.2320\,{\scriptstyle\pm\,0.0053}$ & $\mathbf{0.2930}\,{\scriptstyle\pm\,0.0046}$ & $0.2940\,{\scriptstyle\pm\,0.0023}$ & $\textcolor{secondorange}{-1.07}$ \\
\rowcolor{tablerow}
 & 336 & T5 & 16 & $\mathbf{0.2732}\,{\scriptstyle\pm\,0.0025}$ & $0.2752\,{\scriptstyle\pm\,0.0042}$ & $\mathbf{0.3219}\,{\scriptstyle\pm\,0.0014}$ & $0.3251\,{\scriptstyle\pm\,0.0039}$ & $\textcolor{secondorange}{-0.76}$ \\
\rowcolor{tablerow}
 & 720 & T5 & 16 & $\mathbf{0.3677}\,{\scriptstyle\pm\,0.0100}$ & $0.3698\,{\scriptstyle\pm\,0.0048}$ & $\mathbf{0.3817}\,{\scriptstyle\pm\,0.0056}$ & $0.3823\,{\scriptstyle\pm\,0.0030}$ & $\textcolor{secondorange}{-0.57}$ \\
\addlinespace[2pt]
Electricity & 96 & T0 & 32 & $\mathbf{0.1645}\,{\scriptstyle\pm\,0.0005}$ & $0.1654\,{\scriptstyle\pm\,0.0006}$ & $\mathbf{0.2507}\,{\scriptstyle\pm\,0.0007}$ & $0.2507\,{\scriptstyle\pm\,0.0008}$ & $\textcolor{secondorange}{-0.55}$ \\
 & 192 & T1 & 8 & $\mathbf{0.1770}\,{\scriptstyle\pm\,0.0002}$ & $0.1779\,{\scriptstyle\pm\,0.0013}$ & $0.2619\,{\scriptstyle\pm\,0.0005}$ & $\mathbf{0.2613}\,{\scriptstyle\pm\,0.0015}$ & $\textcolor{secondorange}{-0.51}$ \\
 & 336 & T3 & 8 & $0.1935\,{\scriptstyle\pm\,0.0012}$ & $\mathbf{0.1928}\,{\scriptstyle\pm\,0.0010}$ & $0.2784\,{\scriptstyle\pm\,0.0014}$ & $\mathbf{0.2769}\,{\scriptstyle\pm\,0.0013}$ & $\textcolor{bestblue}{+0.35}$ \\
 & 720 & T1 & 8 & $0.2234\,{\scriptstyle\pm\,0.0014}$ & $\mathbf{0.2231}\,{\scriptstyle\pm\,0.0008}$ & $0.3035\,{\scriptstyle\pm\,0.0015}$ & $\mathbf{0.3023}\,{\scriptstyle\pm\,0.0009}$ & $\textcolor{bestblue}{+0.11}$ \\
\addlinespace[2pt]
\rowcolor{tablerow}
Exchange & 96 & T2 & 32 & $\mathbf{0.0892}\,{\scriptstyle\pm\,0.0032}$ & $0.0930\,{\scriptstyle\pm\,0.0053}$ & $\mathbf{0.2100}\,{\scriptstyle\pm\,0.0036}$ & $0.2148\,{\scriptstyle\pm\,0.0054}$ & $\textcolor{secondorange}{-4.21}$ \\
\rowcolor{tablerow}
 & 192 & T7 & 16 & $0.2139\,{\scriptstyle\pm\,0.0178}$ & $\mathbf{0.2007}\,{\scriptstyle\pm\,0.0207}$ & $0.3325\,{\scriptstyle\pm\,0.0121}$ & $\mathbf{0.3223}\,{\scriptstyle\pm\,0.0162}$ & $\textcolor{bestblue}{+6.18}$ \\
\rowcolor{tablerow}
 & 336 & T5 & 8 & $0.4021\,{\scriptstyle\pm\,0.0474}$ & $\mathbf{0.3637}\,{\scriptstyle\pm\,0.0190}$ & $0.4611\,{\scriptstyle\pm\,0.0278}$ & $\mathbf{0.4410}\,{\scriptstyle\pm\,0.0120}$ & $\textcolor{bestblue}{+9.53}$ \\
\rowcolor{tablerow}
 & 720 & T0 & 16 & $\mathbf{0.9518}\,{\scriptstyle\pm\,0.1153}$ & $0.9732\,{\scriptstyle\pm\,0.1340}$ & $\mathbf{0.7210}\,{\scriptstyle\pm\,0.0369}$ & $0.7328\,{\scriptstyle\pm\,0.0491}$ & $\textcolor{secondorange}{-2.25}$ \\
\addlinespace[2pt]
Traffic & 96 & T0 & 32 & $0.4240\,{\scriptstyle\pm\,0.0029}$ & $\mathbf{0.4216}\,{\scriptstyle\pm\,0.0009}$ & $\mathbf{0.2657}\,{\scriptstyle\pm\,0.0028}$ & $0.2675\,{\scriptstyle\pm\,0.0019}$ & $\textcolor{bestblue}{+0.56}$ \\
 & 192 & T1 & 32 & $0.4431\,{\scriptstyle\pm\,0.0007}$ & $\mathbf{0.4429}\,{\scriptstyle\pm\,0.0004}$ & $\mathbf{0.2733}\,{\scriptstyle\pm\,0.0031}$ & $0.2768\,{\scriptstyle\pm\,0.0012}$ & $\textcolor{bestblue}{+0.04}$ \\
 & 336 & T0 & 32 & $0.4618\,{\scriptstyle\pm\,0.0018}$ & $\mathbf{0.4614}\,{\scriptstyle\pm\,0.0006}$ & $\mathbf{0.2797}\,{\scriptstyle\pm\,0.0024}$ & $0.2817\,{\scriptstyle\pm\,0.0005}$ & $\textcolor{bestblue}{+0.08}$ \\
 & 720 & T7 & 32 & $\mathbf{0.4905}\,{\scriptstyle\pm\,0.0011}$ & $0.4929\,{\scriptstyle\pm\,0.0014}$ & $\mathbf{0.2968}\,{\scriptstyle\pm\,0.0018}$ & $0.2987\,{\scriptstyle\pm\,0.0004}$ & $\textcolor{secondorange}{-0.48}$ \\
\addlinespace[2pt]
\rowcolor{tablerow}
Weather & 96 & T7 & 32 & $0.1553\,{\scriptstyle\pm\,0.0010}$ & $\mathbf{0.1547}\,{\scriptstyle\pm\,0.0015}$ & $0.1964\,{\scriptstyle\pm\,0.0017}$ & $\mathbf{0.1957}\,{\scriptstyle\pm\,0.0027}$ & $\textcolor{bestblue}{+0.38}$ \\
\rowcolor{tablerow}
 & 192 & T5 & 8 & $0.2016\,{\scriptstyle\pm\,0.0014}$ & $\mathbf{0.2008}\,{\scriptstyle\pm\,0.0013}$ & $0.2390\,{\scriptstyle\pm\,0.0009}$ & $\mathbf{0.2374}\,{\scriptstyle\pm\,0.0010}$ & $\textcolor{bestblue}{+0.36}$ \\
\rowcolor{tablerow}
 & 336 & T0 & 16 & $\mathbf{0.2508}\,{\scriptstyle\pm\,0.0012}$ & $0.2514\,{\scriptstyle\pm\,0.0011}$ & $0.2775\,{\scriptstyle\pm\,0.0005}$ & $\mathbf{0.2772}\,{\scriptstyle\pm\,0.0016}$ & $\textcolor{secondorange}{-0.21}$ \\
\rowcolor{tablerow}
 & 720 & T3 & 8 & $\mathbf{0.3233}\,{\scriptstyle\pm\,0.0052}$ & $0.3257\,{\scriptstyle\pm\,0.0056}$ & $\mathbf{0.3278}\,{\scriptstyle\pm\,0.0041}$ & $0.3282\,{\scriptstyle\pm\,0.0032}$ & $\textcolor{secondorange}{-0.74}$ \\
\bottomrule
\end{longtable}
\endgroup

\clearpage
\begingroup

\setlength{\tabcolsep}{4pt}\renewcommand{\arraystretch}{1.00}
\setlength{\LTleft}{\fill}\setlength{\LTright}{\fill}
\setlength{\LTpre}{6pt}\setlength{\LTpost}{6pt}
\begin{longtable}{@{}lrrrrrrr@{}}
\caption{\textbf{TimeMosaic: training and inference costs for the same selected configurations.} Each value averages five runs. Training is measured from steady-state timed epochs after optimizer warm-up; inference is median steady-state full-batch latency. $S=D/U$; values above one favor uniform.}\label{tab:efficiency-timemosaic}\\
\toprule
 &  & \multicolumn{3}{c}{Training (s/epoch) $\downarrow$} & \multicolumn{3}{c}{Inference (ms) $\downarrow$} \\
\cmidrule(lr){3-5}\cmidrule(lr){6-8}
Dataset & $H$ & Adaptive & Uniform & $S_{\mathrm{tr}}$ & Adaptive & Uniform & $S_{\mathrm{inf}}$ \\
\midrule
\endfirsthead
\toprule
 &  & \multicolumn{3}{c}{Training (s/epoch) $\downarrow$} & \multicolumn{3}{c}{Inference (ms) $\downarrow$} \\
\cmidrule(lr){3-5}\cmidrule(lr){6-8}
Dataset & $H$ & Adaptive & Uniform & $S_{\mathrm{tr}}$ & Adaptive & Uniform & $S_{\mathrm{inf}}$ \\
\midrule
\endhead
ETTh1 & 96 & 10.92 & 10.07 & 1.08 & 12.37 & 11.20 & 1.10 \\
 & 192 & 15.15 & 14.29 & 1.06 & 16.63 & 15.42 & 1.08 \\
 & 336 & 8.10 & 7.27 & 1.11 & 9.57 & 8.33 & 1.15 \\
 & 720 & 14.58 & 13.79 & 1.06 & 17.62 & 16.37 & 1.08 \\
\addlinespace[2pt]
\rowcolor{tablerow}
ETTh2 & 96 & 11.19 & 10.41 & 1.08 & 12.33 & 11.14 & 1.11 \\
\rowcolor{tablerow}
 & 192 & 6.29 & 5.45 & 1.16 & 7.34 & 6.15 & 1.19 \\
\rowcolor{tablerow}
 & 336 & 4.63 & 3.84 & 1.21 & 5.79 & 4.61 & 1.26 \\
\rowcolor{tablerow}
 & 720 & 14.97 & 14.17 & 1.06 & 18.11 & 16.92 & 1.07 \\
\addlinespace[2pt]
ETTm1 & 96 & 25.58 & 22.27 & 1.15 & 7.39 & 6.20 & 1.19 \\
 & 192 & 44.08 & 40.65 & 1.08 & 11.85 & 10.65 & 1.11 \\
 & 336 & 18.82 & 15.33 & 1.23 & 5.61 & 4.34 & 1.29 \\
 & 720 & 25.98 & 22.59 & 1.15 & 7.70 & 6.44 & 1.20 \\
\addlinespace[2pt]
\rowcolor{tablerow}
ETTm2 & 96 & 47.47 & 43.89 & 1.08 & 12.86 & 11.61 & 1.11 \\
\rowcolor{tablerow}
 & 192 & 46.33 & 42.95 & 1.08 & 12.32 & 11.12 & 1.11 \\
\rowcolor{tablerow}
 & 336 & 32.44 & 29.09 & 1.12 & 9.15 & 7.95 & 1.15 \\
\rowcolor{tablerow}
 & 720 & 27.26 & 23.99 & 1.14 & 7.99 & 6.80 & 1.17 \\
\addlinespace[2pt]
Electricity & 96 & 509.51 & 506.75 & 1.01 & 75.56 & 74.64 & 1.01 \\
 & 192 & 507.55 & 506.87 & 1.00 & 75.33 & 75.94 & 0.99 \\
 & 336 & 338.71 & 337.36 & 1.00 & 50.95 & 51.03 & 1.00 \\
 & 720 & 496.38 & 495.86 & 1.00 & 76.07 & 76.70 & 0.99 \\
\addlinespace[2pt]
\rowcolor{tablerow}
Exchange & 96 & 10.80 & 10.33 & 1.05 & 22.14 & 20.99 & 1.06 \\
\rowcolor{tablerow}
 & 192 & 9.55 & 9.09 & 1.05 & 16.44 & 15.29 & 1.07 \\
\rowcolor{tablerow}
 & 336 & 2.77 & 2.31 & 1.20 & 5.68 & 4.46 & 1.27 \\
\rowcolor{tablerow}
 & 720 & 6.20 & 5.80 & 1.07 & 12.67 & 11.54 & 1.10 \\
\addlinespace[2pt]
Traffic & 96 & 769.40 & 662.94 & 1.16 & 61.06 & 59.43 & 1.03 \\
 & 192 & 477.26 & 471.18 & 1.01 & 41.24 & 40.47 & 1.02 \\
 & 336 & 320.88 & 315.02 & 1.02 & 28.27 & 27.48 & 1.03 \\
 & 720 & 462.55 & 457.14 & 1.01 & 41.81 & 41.03 & 1.02 \\
\addlinespace[2pt]
\rowcolor{tablerow}
Weather & 96 & 242.02 & 237.96 & 1.02 & 57.03 & 55.82 & 1.02 \\
\rowcolor{tablerow}
 & 192 & 249.42 & 245.62 & 1.02 & 60.50 & 59.29 & 1.02 \\
\rowcolor{tablerow}
 & 336 & 116.74 & 112.72 & 1.04 & 28.99 & 27.74 & 1.05 \\
\rowcolor{tablerow}
 & 720 & 171.07 & 167.22 & 1.02 & 42.59 & 41.36 & 1.03 \\
\bottomrule
\end{longtable}
\endgroup

\section{Model adaptations for the uniform-patch ablations}
\label{app:model-changes}

The controlled study requires a uniform-patch counterpart for each dynamic
method. In all cases, the design principle is the same: modify the patch
generation mechanism while leaving the downstream forecasting backbone,
prediction head, and training budget as close as possible to the original
model. The model implementations used for these ablations are publicly released
open-source research codebases. Our ablation implementations for EntroPE,
HDMixer, and TimeMosaic are provided as supplementary material.

\paragraph{EntroPE.} The EntroPE ablation uses the model's pre-existing static
patching mode with a fixed patch length. Relative to the original method, the
only conceptual change is that entropy-based boundary selection is disabled and
all segments are forced to have equal length. The rest of the forecasting pipeline
remains unchanged. This comparison uses the original implementation's entropy-guided
segmentation and fixed-length alternative within the same architecture.

\paragraph{TimeMosaic.} TimeMosaic ordinarily learns, for each temporal region,
which patch length to use from a small candidate set. The uniform variant removes
this routing decision and applies a single fixed patch length across all regions.
The fixed-length variant reuses the same
patch-embedding operator associated with that patch size in the original
multi-scale embedding bank rather than introducing a new projection module. The
downstream encoder and prediction head are unchanged, while the auxiliary routing
regularization is disabled because no patch-size classification is performed.
The uniform implementation repeat-pads its patch tokens to the adaptive
materialized token width, preserving the forecast-head input width.
\framingrevision{This keeps that tensor dimension fixed while the token
representations retain the content of their respective patch lengths.}

\paragraph{HDMixer.} HDMixer ordinarily uses deformable patch extraction, in
which the model predicts patch offsets and widths and applies an auxiliary
patch-entropy regularizer. The uniform variant replaces this deformable extraction
stage with standard non-overlapping fixed-length patches. Under this ablation,
the deformable localization mechanism and the associated auxiliary regularizer are
inactive. \framingrevision{The mixer architecture is retained, and its native
flattened forecast-head width varies with the number of patches, allowing
parameter count and representation capacity to vary with the selected patch
length.} The number of patches is recomputed from the chosen fixed length so
that the downstream forecast module receives a tensorization matched to the
selected granularity.

\paragraph{Common interpretation.} The ablations retain each method's
architectural family and shared training settings while replacing adaptive
patch selection and its associated losses. Downstream tensors follow each
method's patch-count convention, so token counts and parameter counts can
differ between variants. \framingrevision{The comparison evaluates the accuracy
and execution cost of the resulting adaptive and tuned uniform systems at
their respective resource requirements.}

\section{Experimental hyperparameters and hardware}
\label{app:hyper}
\definecolor{ProtocolRow}{HTML}{F2F5F8}

\paragraph{Validation-only selection and independent evaluation.}
For each method, arm, dataset, and forecast horizon, every candidate is trained
with three independent tuning seeds. We rank candidates by their mean best
validation MSE; the predeclared candidate order breaks exact ties. The selector
has no access to test metrics. The selected adaptive and uniform configurations
are then trained with freshly initialized forecasting networks using five new
seeds, paired across the two arms. EntroPE retains its frozen pretrained
entropy estimator; that auxiliary model is not retrained in these runs.
Test MSE and MAE average those five final runs. This is an expanded adaptive
search: the adaptive arm receives more candidate configurations than the
uniform arm, as summarized in \Cref{tab:selection-budget}. Across eight
datasets and four horizons, the
reported comparison uses 4,800 tuning runs and 960 selected final runs.

\begin{table}[!htbp]
\centering

\setlength{\tabcolsep}{3pt}
\renewcommand{\arraystretch}{1.00}
\caption{\textbf{Search budget and final evaluation.} Each candidate has three tuning runs; each selected arm has five final runs.}
\label{tab:selection-budget}
\begin{tabularx}{\linewidth}{@{}>{\raggedright\arraybackslash}X>{\raggedleft\arraybackslash}X>{\raggedleft\arraybackslash}X>{\raggedleft\arraybackslash}X>{\raggedleft\arraybackslash}X@{}}
\toprule
\textbf{Method} & \textbf{Adaptive} & \textbf{Uniform} & \textbf{Tuning runs} & \textbf{Final runs} \\
\midrule
\rowcolor{ProtocolRow} EntroPE & 12 & 5 & 1,632 & 320 \\
HDMixer & 16 & 5 & 2,016 & 320 \\
\rowcolor{ProtocolRow} TimeMosaic & 9 & 3 & 1,152 & 320 \\
\bottomrule
\end{tabularx}
\end{table}

\paragraph{Search spaces.}
For EntroPE, $q$ is the empirical entropy quantile used to set the boundary
threshold, and $L_{\max}$ is the maximum patch length. The adaptive grid is
$q\in\{0.55,0.75,0.95\}$ crossed with $L_{\max}\in\{8,16,24,32\}$; consecutive
boundaries are disallowed. The uniform grid is $p\in\{4,8,16,24,32\}$.
TimeMosaic varies router temperature $\tau\in\{0.25,0.5,1\}$ and balance-loss
weight $\lambda_{\rm bal}\in\{10^{-4},10^{-3},10^{-2}\}$, keeping the patch bank
$\{8,16,32\}$ fixed. Its uniform grid is $p\in\{8,16,32\}$.
Uniform candidate P$p$ denotes fixed patch length $p$. The complete adaptive
candidate grids are listed in \Cref{tab:entrope-timemosaic-search}.

\begin{table}[!htbp]
\centering

\setlength{\tabcolsep}{3pt}
\renewcommand{\arraystretch}{1.00}
\caption{\textbf{EntroPE and TimeMosaic adaptive candidates.} IDs are used in the selected-result tables.}
\label{tab:entrope-timemosaic-search}
\begin{tabularx}{\linewidth}{@{}>{\raggedright\arraybackslash}X>{\raggedleft\arraybackslash}X>{\raggedleft\arraybackslash}X>{\raggedright\arraybackslash}X>{\raggedleft\arraybackslash}X>{\raggedleft\arraybackslash}X@{}}
\toprule
\textbf{EntroPE} & \textbf{$q$} & \textbf{$L_{\max}$} & \textbf{TimeMosaic} & \textbf{$\tau$} & \textbf{$\lambda_{\rm bal}$} \\
\midrule
\rowcolor{ProtocolRow} E0 & 0.55 & 8 & T0 & 0.25 & $10^{-4}$ \\
\rowcolor{ProtocolRow} E1 & 0.55 & 16 & T1 & 0.25 & $10^{-3}$ \\
\rowcolor{ProtocolRow} E2 & 0.55 & 24 & T2 & 0.25 & $10^{-2}$ \\
\rowcolor{ProtocolRow} E3 & 0.55 & 32 & T3 & 0.5 & $10^{-4}$ \\
\addlinespace[1pt]
E4 & 0.75 & 8 & T4 & 0.5 & $10^{-3}$ \\
E5 & 0.75 & 16 & T5 & 0.5 & $10^{-2}$ \\
E6 & 0.75 & 24 & T6 & 1 & $10^{-4}$ \\
E7 & 0.75 & 32 & T7 & 1 & $10^{-3}$ \\
\addlinespace[1pt]
\rowcolor{ProtocolRow} E8 & 0.95 & 8 & T8 & 1 & $10^{-2}$ \\
\rowcolor{ProtocolRow} E9 & 0.95 & 16 &  &  &  \\
\rowcolor{ProtocolRow} E10 & 0.95 & 24 &  &  &  \\
\rowcolor{ProtocolRow} E11 & 0.95 & 32 &  &  &  \\
\bottomrule
\end{tabularx}
\end{table}

For HDMixer, the first twelve candidates cross patch length
$p\in\{8,16,24,32\}$ with deformation range $d\in\{0.125,0.25,0.5\}$,
using stride $s=p/2$, length-regularization weight $\lambda=0.1$, and
radius parameter $r=0.01$. Four additional candidates vary $\lambda$ or $r$
around the $p/s=16/8$, $d=0.25$ configuration. The uniform grid is
$p\in\{4,8,16,24,32\}$; the complete adaptive grid is listed in
\Cref{tab:hdmixer-search}.

\begin{table}[!htbp]
\centering

\setlength{\tabcolsep}{3pt}
\renewcommand{\arraystretch}{1.00}
\caption{\textbf{Complete HDMixer adaptive search.} All sixteen candidates are evaluated independently for every dataset and horizon.}
\label{tab:hdmixer-search}
\begin{tabularx}{\linewidth}{@{}>{\raggedright\arraybackslash}X>{\raggedleft\arraybackslash}X>{\raggedleft\arraybackslash}X>{\raggedleft\arraybackslash}X>{\raggedleft\arraybackslash}X>{\raggedleft\arraybackslash}X@{}}
\toprule
\textbf{ID} & \textbf{$p$} & \textbf{$s$} & \textbf{$d$} & \textbf{$\lambda$} & \textbf{$r$} \\
\midrule
\rowcolor{ProtocolRow} H0 & 8 & 4 & 0.125 & 0.1 & 0.01 \\
\rowcolor{ProtocolRow} H1 & 8 & 4 & 0.25 & 0.1 & 0.01 \\
\rowcolor{ProtocolRow} H2 & 8 & 4 & 0.5 & 0.1 & 0.01 \\
\addlinespace[1pt]
H3 & 16 & 8 & 0.125 & 0.1 & 0.01 \\
H4 & 16 & 8 & 0.25 & 0.1 & 0.01 \\
H5 & 16 & 8 & 0.5 & 0.1 & 0.01 \\
\addlinespace[1pt]
\rowcolor{ProtocolRow} H6 & 24 & 12 & 0.125 & 0.1 & 0.01 \\
\rowcolor{ProtocolRow} H7 & 24 & 12 & 0.25 & 0.1 & 0.01 \\
\rowcolor{ProtocolRow} H8 & 24 & 12 & 0.5 & 0.1 & 0.01 \\
\addlinespace[1pt]
H9 & 32 & 16 & 0.125 & 0.1 & 0.01 \\
H10 & 32 & 16 & 0.25 & 0.1 & 0.01 \\
H11 & 32 & 16 & 0.5 & 0.1 & 0.01 \\
\addlinespace[1pt]
\rowcolor{ProtocolRow} H12 & 16 & 8 & 0.25 & 0.01 & 0.01 \\
\rowcolor{ProtocolRow} H13 & 16 & 8 & 0.25 & 1 & 0.01 \\
\rowcolor{ProtocolRow} H14 & 16 & 8 & 0.25 & 0.1 & 0.001 \\
\rowcolor{ProtocolRow} H15 & 16 & 8 & 0.25 & 0.1 & 0.1 \\
\bottomrule
\end{tabularx}
\end{table}

\paragraph{Fixed architecture and optimization settings.}
Only the patch/router parameters listed above are searched. Non-patching
settings, preprocessing, optimizer, epoch cap, and stopping rule remain fixed
within each method--dataset--horizon comparison. All three methods use Adam and
validation-based early stopping, restoring the checkpoint with the lowest
validation MSE. Forecaster
training disables automatic mixed precision. TimeMosaic's uniform counterpart preserves the adaptive
model's materialized token/head width. \framingrevision{HDMixer retains its
native geometry, so patch length can change token count, prediction-head width,
parameter count, and compute.}

The effective fixed settings are given in
\Cref{tab:entrope-hparams,tab:hdmixer-hparams,tab:timemosaic-hparams,tab:timemosaic-hparams-other}.
$T$ is lookback length, $L$ is encoder depth, $h$ is attention-head count, and
$B$ is batch size.
The training caps are 20 epochs for EntroPE, 50 for HDMixer, and 10 for
TimeMosaic; these are caps, not the number actually completed by every run.

\begin{table}[!htbp]
\centering

\setlength{\tabcolsep}{3pt}
\renewcommand{\arraystretch}{1.00}
\caption{\textbf{EntroPE fixed configurations.} Shared by all four horizons and both arms; $T=96$, patience $=10$, and learning-rate schedule \texttt{TST}.}
\label{tab:entrope-hparams}
\begin{tabularx}{\linewidth}{@{}>{\raggedright\arraybackslash}X>{\raggedleft\arraybackslash}X>{\raggedleft\arraybackslash}X>{\raggedleft\arraybackslash}X>{\raggedleft\arraybackslash}X>{\raggedleft\arraybackslash}X>{\raggedleft\arraybackslash}X>{\raggedleft\arraybackslash}X@{}}
\toprule
\textbf{Dataset} & \textbf{$d_{\rm model}$} & \textbf{$d_{\rm ff}$} & \textbf{$h$} & \textbf{$L$} & \textbf{$B$} & \textbf{LR} & \textbf{Dropout} \\
\midrule
\rowcolor{ProtocolRow} ETTh1 & 8 & 2048 & 2 & 1 & 64 & $10^{-3}$ & 0.05 \\
ETTh2 & 8 & 2048 & 2 & 2 & 64 & $10^{-2}$ & 0.1 \\
\rowcolor{ProtocolRow} ETTm1 & 16 & 2048 & 2 & 1 & 32 & $10^{-2}$ & 0.1 \\
ETTm2 & 16 & 2048 & 4 & 1 & 32 & $10^{-3}$ & 0.1 \\
\rowcolor{ProtocolRow} Electricity & 32 & 2048 & 4 & 2 & 32 & $10^{-2}$ & 0.1 \\
Traffic & 8 & 256 & 2 & 3 & 8 & $10^{-2}$ & 0.1 \\
\rowcolor{ProtocolRow} Weather & 16 & 2048 & 2 & 2 & 128 & $10^{-2}$ & 0.2 \\
Exchange & 8 & 256 & 2 & 3 & 32 & $10^{-2}$ & 0.1 \\
\bottomrule
\end{tabularx}
\par\vspace{3pt}\begin{minipage}{\linewidth} \texttt{fc\_dropout} and \texttt{head\_dropout} are both 0.1. Traffic and Exchange use $d_{\rm ff}=256$; the other six datasets use $d_{\rm ff}=2{,}048$.\end{minipage}
\end{table}

\begin{table}[!htbp]
\centering

\setlength{\tabcolsep}{3pt}
\renewcommand{\arraystretch}{1.00}
\caption{\textbf{HDMixer fixed configurations.} Both arms use $T=336$, $d_{\rm model}=16$, $d_{\rm ff}=32$, $h=4$, and learning-rate schedule \texttt{type3}.}
\label{tab:hdmixer-hparams}
\begin{tabularx}{\linewidth}{@{}>{\raggedright\arraybackslash}X>{\raggedleft\arraybackslash}X>{\raggedleft\arraybackslash}X>{\raggedleft\arraybackslash}X>{\raggedleft\arraybackslash}X>{\raggedleft\arraybackslash}X>{\raggedleft\arraybackslash}X>{\raggedleft\arraybackslash}X@{}}
\toprule
\textbf{Dataset} & \textbf{$H$} & \textbf{$L$} & \textbf{$B$} & \textbf{LR} & \textbf{Drop.} & \textbf{FC drop.} & \textbf{Patience} \\
\midrule
\rowcolor{ProtocolRow} ETTh1 & 96,192 & 1 & 256 & $5\!\times\!10^{-4}$ & 0.8 & 0.3 & 20 \\
\rowcolor{ProtocolRow} ETTh1 & 336,720 & 2 & 512 & $5\!\times\!10^{-3}$ & 0.8 & 0.3 & 20 \\
\addlinespace[1pt]
ETTh2 & 96 & 1 & 256 & $5\!\times\!10^{-4}$ & 0.8 & 0.3 & 10 \\
ETTh2 & 192,720 & 1 & 512 & $5\!\times\!10^{-4}$ & 0.8 & 0.3 & 10 \\
ETTh2 & 336 & 1 & 512 & $10^{-3}$ & 0.8 & 0.3 & 10 \\
\addlinespace[1pt]
\rowcolor{ProtocolRow} ETTm1 & All & 1 & 256 & $10^{-3}$ & 0.8 & 0.3 & 20 \\
\addlinespace[1pt]
ETTm2 & All & 1 & 256 & $5\!\times\!10^{-4}$ & 0.8 & 0.3 & 20 \\
\addlinespace[1pt]
\rowcolor{ProtocolRow} Electricity & All & 2 & 32 & $10^{-3}$ & 0.8 & 0.3 & 20 \\
\addlinespace[1pt]
Traffic & All & 2 & 8 & $10^{-3}$ & 0.8 & 0.3 & 20 \\
\addlinespace[1pt]
\rowcolor{ProtocolRow} Weather & All & 2 & 32 & $5\!\times\!10^{-4}$ & 0.3 & 0.1 & 20 \\
\addlinespace[1pt]
Exchange & All & 2 & 32 & $5\!\times\!10^{-4}$ & 0.8 & 0.3 & 20 \\
\bottomrule
\end{tabularx}
\par\vspace{3pt}\begin{minipage}{\linewidth} \texttt{head\_dropout} is 0 throughout. All denotes all four horizons.\end{minipage}
\end{table}

\begin{table}[H]
\centering

\setlength{\tabcolsep}{3pt}
\renewcommand{\arraystretch}{1.00}
\caption{\textbf{TimeMosaic fixed configurations: ETT datasets.} Both arms use $d_{\rm model}=512$, LR $=10^{-4}$, patience $=3$, and learning-rate schedule \texttt{type1}.}
\label{tab:timemosaic-hparams}
\begin{tabularx}{\linewidth}{@{}>{\raggedright\arraybackslash}X>{\raggedleft\arraybackslash}X>{\raggedleft\arraybackslash}X>{\raggedleft\arraybackslash}X>{\raggedleft\arraybackslash}X>{\raggedleft\arraybackslash}X>{\raggedleft\arraybackslash}X>{\raggedright\arraybackslash}X@{}}
\toprule
\textbf{Dataset} & \textbf{$H$} & \textbf{$T$} & \textbf{$d_{\rm ff}$} & \textbf{$L$} & \textbf{$h$} & \textbf{$B$} & \textbf{Channel} \\
\midrule
\rowcolor{ProtocolRow} ETTh1 & 96 & 320 & 2048 & 2 & 8 & 32 & CI \\
\rowcolor{ProtocolRow} ETTh1 & 192 & 320 & 2048 & 3 & 1 & 32 & CI \\
\rowcolor{ProtocolRow} ETTh1 & 336 & 320 & 2048 & 2 & 16 & 32 & CDA \\
\rowcolor{ProtocolRow} ETTh1 & 720 & 320 & 2048 & 3 & 2 & 32 & CD \\
\addlinespace[1pt]
ETTh2 & 96 & 320 & 2048 & 2 & 1 & 32 & CDA \\
ETTh2 & 192 & 320 & 2048 & 1 & 1 & 32 & CDA \\
ETTh2 & 336 & 320 & 2048 & 1 & 2 & 32 & CDA \\
ETTh2 & 720 & 320 & 2048 & 3 & 16 & 32 & CI \\
\addlinespace[1pt]
\rowcolor{ProtocolRow} ETTm1 & 96 & 320 & 2048 & 1 & 8 & 32 & CI \\
\rowcolor{ProtocolRow} ETTm1 & 192 & 320 & 2048 & 2 & 1 & 32 & CI \\
\rowcolor{ProtocolRow} ETTm1 & 336 & 320 & 2048 & 1 & 3 & 32 & CI \\
\rowcolor{ProtocolRow} ETTm1 & 720 & 320 & 2048 & 1 & 4 & 32 & CD \\
\addlinespace[1pt]
ETTm2 & 96 & 320 & 2048 & 2 & 2 & 32 & CDA \\
ETTm2 & 192 & 320 & 2048 & 2 & 1 & 32 & CDA \\
ETTm2 & 336 & 320 & 2048 & 2 & 16 & 32 & CI \\
ETTm2 & 720 & 320 & 2048 & 1 & 2 & 32 & CDA \\
\bottomrule
\end{tabularx}
\par\vspace{3pt}\begin{minipage}{\linewidth} Dropout and \texttt{fc\_dropout} are both 0.1. Channel labels preserve the executed codebase modes.\end{minipage}
\end{table}

\begin{table}[!htbp]
\centering

\setlength{\tabcolsep}{3pt}
\renewcommand{\arraystretch}{1.00}
\caption{\textbf{TimeMosaic fixed configurations: remaining datasets.} Both arms use $d_{\rm model}=512$, LR $=10^{-4}$, patience $=3$, and learning-rate schedule \texttt{type1}.}
\label{tab:timemosaic-hparams-other}
\begin{tabularx}{\linewidth}{@{}>{\raggedright\arraybackslash}X>{\raggedleft\arraybackslash}X>{\raggedleft\arraybackslash}X>{\raggedleft\arraybackslash}X>{\raggedleft\arraybackslash}X>{\raggedleft\arraybackslash}X>{\raggedleft\arraybackslash}X>{\raggedright\arraybackslash}X@{}}
\toprule
\textbf{Dataset} & \textbf{$H$} & \textbf{$T$} & \textbf{$d_{\rm ff}$} & \textbf{$L$} & \textbf{$h$} & \textbf{$B$} & \textbf{Channel} \\
\midrule
\rowcolor{ProtocolRow} Electricity & 96 & 96 & 2048 & 2 & 8 & 16 & CDP \\
\rowcolor{ProtocolRow} Electricity & 192 & 96 & 2048 & 2 & 8 & 16 & CDP \\
\rowcolor{ProtocolRow} Electricity & 336 & 96 & 2048 & 2 & 8 & 16 & CI+ \\
\rowcolor{ProtocolRow} Electricity & 720 & 96 & 2048 & 2 & 8 & 16 & CDP \\
\addlinespace[1pt]
Traffic & 96 & 96 & 512 & 2 & 8 & 4 & CDP \\
Traffic & 192 & 96 & 512 & 2 & 8 & 4 & CDP \\
Traffic & 336 & 96 & 512 & 2 & 8 & 4 & CDP \\
Traffic & 720 & 96 & 512 & 2 & 8 & 4 & CDP \\
\addlinespace[1pt]
\rowcolor{ProtocolRow} Weather & 96 & 320 & 2048 & 3 & 1 & 32 & CDA \\
\rowcolor{ProtocolRow} Weather & 192 & 320 & 2048 & 3 & 4 & 32 & CDA \\
\rowcolor{ProtocolRow} Weather & 336 & 320 & 2048 & 2 & 8 & 32 & CDA \\
\rowcolor{ProtocolRow} Weather & 720 & 320 & 2048 & 2 & 8 & 32 & CDA \\
\addlinespace[1pt]
Exchange & 96 & 320 & 2048 & 3 & 16 & 32 & CI \\
Exchange & 192 & 320 & 2048 & 3 & 1 & 32 & CDA \\
Exchange & 336 & 320 & 2048 & 1 & 4 & 32 & CDA \\
Exchange & 720 & 320 & 2048 & 2 & 2 & 32 & CDA \\
\bottomrule
\end{tabularx}
\par\vspace{3pt}\begin{minipage}{\linewidth} Dropout and \texttt{fc\_dropout} are both 0.1. Channel labels preserve the executed codebase modes.\end{minipage}
\end{table}

\paragraph{Training and inference measurements.}
All selected checkpoints are profiled on NVIDIA H200 GPUs with paired arms on
the same batch size. Training cost is measured from complete timed
optimization epochs after warm-up and reported in seconds per epoch and per
step. It excludes validation and checkpoint management and does not measure
time to convergence. Inference is steady-state full-batch forward latency with
inputs resident on the GPU and synchronized CUDA-event timing. The reported
per-run value is the median over repeated measurements.

Native adaptive timing includes route construction and, for EntroPE, the
frozen entropy estimator. \framingrevision{To estimate routing overhead, we first run the adaptive model
and save the patch tokens it produces. We then time a second run with the same
checkpoint and inputs, bypassing adaptive token generation and feeding the
saved tokens to the forecasting backbone. Routing overhead is the normal
adaptive run time minus the time of this second run, expressed as a percentage
of the normal adaptive run time.}
Dataset-level figures average the four horizon means equally. Speedup is
adaptive cost divided by uniform cost, so values above one favor uniform
patching.

\section{Benchmark dataset notes}
\label{app:datasets}

The main experiments use eight standard multivariate long-horizon forecasting
benchmarks. The suite combines the ETT variants and Weather from the
Informer benchmark~\citep{zhou2021informer} with Electricity, Exchange, and
Traffic from the LSTNet benchmark collection~\citep{lai2018modeling}.
Across HDMixer, EntroPE, and TimeMosaic, the ETT loaders use the first 12, next
4, and following 4 months for training, validation, and testing, respectively,
with 30-day months. The custom-data loaders use chronological 70\%/10\%/20\%
splits for Weather, Electricity, Traffic, and Exchange. Standardization is fit
on the training segment only.

\paragraph{Dataset descriptions.}
\textbf{Electricity} records hourly electricity consumption for 321 clients
and is used for multivariate load forecasting. \textbf{ETTh1},
\textbf{ETTh2}, \textbf{ETTm1}, and \textbf{ETTm2} track electricity
transformer temperature and load, with the `h' variants sampled hourly and
the `m' variants sampled every 15 minutes. \textbf{Exchange} contains daily
exchange rates for multiple currencies relative to the U.S. dollar.
\textbf{Traffic} records hourly road occupancy measurements from a freeway
sensor network. \textbf{Weather} contains 10-minute atmospheric
measurements from a weather station.

Table~\ref{tab:dataset-summary} reports raw time observations assigned to each
split, before constructing forecast windows. Given lookback $T$ and horizon
$H$, the training window count is $n_{\rm train}-T-H+1$; validation and test
use the preceding $T$ observations as context and have $n_{\rm split}-H+1$
windows. Thus window counts depend on the method and horizon. The native
EntroPE and HDMixer loaders discard incomplete final batches; TimeMosaic keeps
them. These loader choices are identical between the adaptive and uniform arms
of each method.

\begin{table}[H]
\centering

\setlength{\tabcolsep}{3pt}
\renewcommand{\arraystretch}{1.12}
\caption{\textbf{Benchmark dimensions and chronological splits.} Counts are raw observations per split, not forecast windows. ETT uses the fixed 20-month benchmark interval.}
\label{tab:dataset-summary}
\begin{tabularx}{\linewidth}{@{}>{\raggedright\arraybackslash}X>{\raggedleft\arraybackslash}X>{\raggedleft\arraybackslash}X>{\raggedleft\arraybackslash}X>{\raggedleft\arraybackslash}X>{\raggedright\arraybackslash}X@{}}
\toprule
\textbf{Dataset} & \textbf{Dim.} & \textbf{Train} & \textbf{Validation} & \textbf{Test} & \textbf{Frequency} \\
\midrule
\rowcolor{ProtocolRow} ETTh1 & 7 & 8,640 & 2,880 & 2,880 & Hourly \\
ETTh2 & 7 & 8,640 & 2,880 & 2,880 & Hourly \\
\rowcolor{ProtocolRow} ETTm1 & 7 & 34,560 & 11,520 & 11,520 & 15 min \\
ETTm2 & 7 & 34,560 & 11,520 & 11,520 & 15 min \\
\rowcolor{ProtocolRow} Electricity & 321 & 18,412 & 2,632 & 5,260 & Hourly \\
Traffic & 862 & 12,280 & 1,756 & 3,508 & Hourly \\
\rowcolor{ProtocolRow} Weather & 21 & 36,887 & 5,270 & 10,539 & 10 min \\
Exchange & 8 & 5,311 & 760 & 1,517 & Daily \\
\bottomrule
\end{tabularx}
\end{table}

\paragraph{Hardware.}
\framingrevision{All experiments were conducted on NVIDIA H200 GPUs with
141\,GB of VRAM each.}

\end{document}